\documentclass[11pt]{article}

\usepackage[final]{acl}

\usepackage{times}
\usepackage{latexsym}
\usepackage[T1]{fontenc}
\usepackage[utf8]{inputenc}

\usepackage{microtype}
\usepackage{inconsolata}
\usepackage{graphicx}
\usepackage{subcaption}
\usepackage{booktabs} %

\usepackage{tcolorbox}
\usepackage{cancel}
\usepackage{placeins}

\usepackage{algorithm}
\usepackage{algpseudocode}
\usepackage{enumitem}
\usepackage{amsmath}
\newcommand{\probP}{\text{I\kern-0.15em P}}

\usepackage{float}
\usepackage{tikz}
\usetikzlibrary{arrows.meta, decorations.pathreplacing}

\usepackage{amssymb}
\usepackage{mathtools}
\usepackage{amsthm}

\usepackage[capitalize,noabbrev]{cleveref}

\theoremstyle{plain}
\newtheorem{theorem}{Theorem}[section]
\newtheorem{proposition}[theorem]{Proposition}
\newtheorem{lemma}[theorem]{Lemma}

\theoremstyle{definition}

\newtheorem{assumption}[theorem]{Assumption}
\theoremstyle{remark}
\newtheorem{remark}[theorem]{Remark}

\usepackage[textsize=tiny]{todonotes}

\title{Selective Regenerative Decoding: Trajectory-Level Intervention for Inference-Time Reasoning}

\author{
  \textbf{Sophia Xiao Pu}\textsuperscript{1}\thanks{\;\,Work done as a research intern at Amazon.},
  \textbf{Yumo Xu}\textsuperscript{2}\thanks{\;\,Work done while at Amazon.},
  \textbf{Sailik Sengupta}\textsuperscript{3}\thanks{\;Correspondence: \href{mailto:sailiks@amazon.com}{\tt sailiks@amazon.com}},
  \textbf{Millennium Bismay}\textsuperscript{3},
\\
  \textbf{Ruixue Lian}\textsuperscript{4}\footnotemark[2],
  \textbf{James Gung}\textsuperscript{3},
  \textbf{Yi-an Lai}\textsuperscript{3},
  \textbf{Arshit Gupta}\textsuperscript{3}
\\
  \begingroup
  \small
  \textsuperscript{1}University of California, Santa Barbara,
  \textsuperscript{2}Netflix,
  \textsuperscript{3}Amazon Science,
  \textsuperscript{4}Meta
  \endgroup
}

\begin{document}
\maketitle

\begin{abstract}

Inference-time decoding methods improve LLM reasoning by exploring multiple candidate trajectories, yet treat each trajectory as atomic---either retaining it whole or discarding it irreversibly. This wastes computation on partially promising candidates whose high-quality prefixes are abandoned alongside degraded suffixes. We introduce Selective Regenerative Decoding (SRD), which routes each candidate to discard, keep, or refine only the degraded portion of suffix while preserving all the useful prefix of borderline candidates --- without requiring a larger target model. Under mild assumptions, SRD achieves a provable $1.28$--$1.36\times$ sample efficiency gain over rejection sampling with strictly higher expected trajectory quality, with the gain growing as the candidate pool grows. Across MATH500, GPQA Diamond, HotpotQA, and AlpacaEval with multiple generation--reward model pairs, SRD matches Best-of-$N$ accuracy with substantially fewer generated tokens and outperforms speculative rejection in low-compute regimes. By enabling segment-level intervention rather than whole-trajectory selection, SRD opens a previously underexplored region of the accuracy--compute tradeoff for inference-time reasoning.

\end{abstract}

\section{Introduction}
\label{sec:introduction}
\begin{figure}[!t]
\centering
\resizebox{0.92\linewidth}{!}{%
\begin{tikzpicture}[
    every node/.style={font=\small},
]

\definecolor{cblue}{RGB}{68,119,170}
\definecolor{corange}{RGB}{204,102,0}
\definecolor{cteal}{RGB}{0,153,136}
\definecolor{cgray}{RGB}{210,210,210}

\tikzset{segbase/.style={rounded corners=0.8pt, line width=0.25pt}}

\def\segW{0.72}   %
\def\segH{0.30}   %
\def\sp{0.78}     %

\newcommand{\seg}[3]{%
  \fill[#3, draw=#3!50!black, segbase] (#1,#2) rectangle ++(\segW,\segH);%
}
\newcommand{\segkey}[3]{%
  \fill[#3, draw=#3!50!black, segbase] (#1,#2) rectangle ++(\segW/2,\segH);%
}
\newcommand{\noseg}[2]{%
  \fill[cgray, draw=gray!50, segbase, densely dashed] (#1,#2) rectangle ++(\segW,\segH);%
}

\newcommand{\yc}[1]{\pgfmathparse{#1+\segH/2}\pgfmathresult}

\begin{scope}[yshift=5.6cm]
  \node[font=\footnotesize\bfseries, anchor=south west] at (-0.15, 1.90)
    {\strut Best-of-$N$};
  \draw[gray!40, line width=0.4pt] (-0.15, 1.87) -- (5.6, 1.87);

  \node[font=\tiny, text=black!70, anchor=east] at (-0.08, {1.35+\segH/2}) {$\tau_1$};
  \foreach \i in {0,...,4} { \seg{0+\i*\sp}{1.35}{cblue} }
  \node[font=\scriptsize, green!60!black, anchor=west] at (4.1, {1.35+\segH/2}) {\checkmark};
  \node[font=\tiny, text=black!60, anchor=west] at (4.35, {1.35+\segH/2}) {Select};

  \node[font=\tiny, text=black!70, anchor=east] at (-0.08, {0.75+\segH/2}) {$\tau_2$};
  \foreach \i in {0,...,2} { \seg{0+\i*\sp}{0.75}{cblue} }
  \foreach \i in {3,4} { \seg{0+\i*\sp}{0.75}{corange} }
  \node[font=\scriptsize, red!60!black, anchor=west] at (4.1, {0.75+\segH/2}) {$\times$};
  \node[font=\tiny, text=black!60, anchor=west] at (4.35, {0.75+\segH/2}) {Discard};

  \node[font=\tiny, text=black!70, anchor=east] at (-0.08, {0.15+\segH/2}) {$\tau_3$};
  \foreach \i in {0,...,4} { \seg{0+\i*\sp}{0.15}{corange} }
  \node[font=\scriptsize, red!60!black, anchor=west] at (4.1, {0.15+\segH/2}) {$\times$};
  \node[font=\tiny, text=black!60, anchor=west] at (4.35, {0.15+\segH/2}) {Discard};
\end{scope}

\begin{scope}[yshift=2.8cm]
  \node[font=\footnotesize\bfseries, anchor=south west] at (-0.15, 1.90)
    {\strut Speculative Rejection};
  \draw[gray!40, line width=0.4pt] (-0.15, 1.87) -- (5.6, 1.87);

  \node[font=\tiny, text=black!70, anchor=east] at (-0.08, {1.35+\segH/2}) {$\tau_1$};
  \foreach \i in {0,...,4} { \seg{0+\i*\sp}{1.35}{cblue} }
  \node[font=\scriptsize, green!60!black, anchor=west] at (4.1, {1.35+\segH/2}) {\checkmark};
  \node[font=\tiny, text=black!60, anchor=west] at (4.35, {1.35+\segH/2}) {Accept};

  \node[font=\tiny, text=black!70, anchor=east] at (-0.08, {0.6+\segH/2}) {$\tau_2$};
  \foreach \i in {0,...,2} { \seg{0+\i*\sp}{0.6}{cblue} }
  \seg{0+3*\sp}{0.6}{corange}
  \noseg{0+4*\sp}{0.6}
  \node[font=\scriptsize, red!60!black, anchor=west] at (4.1, {0.6+\segH/2}) {$\times$};
  \node[font=\tiny, text=black!60, anchor=west] at (4.35, {0.6+\segH/2}) {Reject};

  \draw[decorate, decoration={brace, amplitude=2.5pt}, red!50!black, line width=0.3pt]
    (0, {0.75+\segH-0.08}) -- ({2*\sp+\segW}, {0.75+\segH-0.08});
  \node[font=\tiny, text=red!50!black, anchor=south]
    at ({1*\sp+\segW/2}, {0.75+\segH-0.12}) {useful prefix lost};

  \node[font=\tiny, text=black!70, anchor=east] at (-0.08, {0.15+\segH/2}) {$\tau_3$};
  \seg{0}{0.15}{corange}
  \foreach \i in {1,...,4} { \noseg{0+\i*\sp}{0.15} }
  \node[font=\scriptsize, red!60!black, anchor=west] at (4.1, {0.15+\segH/2}) {$\times$};
  \node[font=\tiny, text=black!60, anchor=west] at (4.35, {0.15+\segH/2}) {Reject};
\end{scope}

\begin{scope}[yshift=0cm]
  \node[font=\footnotesize\bfseries, anchor=south west] at (-0.15, 1.90)
    {\strut SRD (Ours)};
  \draw[gray!40, line width=0.4pt] (-0.15, 1.87) -- (5.6, 1.87);

  \node[font=\tiny, text=black!70, anchor=east] at (-0.08, {1.35+\segH/2}) {$\tau_1$};
  \foreach \i in {0,...,4} { \seg{0+\i*\sp}{1.35}{cblue} }
  \node[font=\scriptsize, green!60!black, anchor=west] at (4.1, {1.35+\segH/2}) {\checkmark};
  \node[font=\tiny, text=black!60, anchor=west] at (4.35, {1.35+\segH/2}) {Keep};

  \node[font=\tiny, text=black!70, anchor=east] at (-0.08, {0.9+\segH/2}) {$\tau_2$};
  \foreach \i in {0,...,2} { \seg{0+\i*\sp}{0.9}{cblue} }
  \foreach \i in {3,4} { \seg{0+\i*\sp}{0.9}{cteal} }
  \node[font=\scriptsize, green!60!black, anchor=west] at (4.1, {0.9+\segH/2}) {\checkmark};
  \node[font=\tiny, text=black!60, anchor=west] at (4.35, {0.9+\segH/2}) {Refine};

  \draw[decorate, decoration={brace, amplitude=2.5pt, mirror}, cteal!70!black, line width=0.3pt]
    ({3*\sp}, {0.85}) -- ({4*\sp+\segW}, {0.85});
  \node[font=\tiny, text=cteal!70!black, anchor=north]
    at ({3.5*\sp+\segW/2}, {0.82}) {regenerated};

  \node[font=\tiny, text=black!70, anchor=east] at (-0.08, {0.15+\segH/2}) {$\tau_3$};
  \foreach \i in {0,...,4} { \seg{0+\i*\sp}{0.15}{corange} }
  \node[font=\scriptsize, red!60!black, anchor=west] at (4.1, {0.15+\segH/2}) {$\times$};
  \node[font=\tiny, text=black!60, anchor=west] at (4.35, {0.15+\segH/2}) {Discard};
\end{scope}

\begin{scope}[yshift=-0.5cm]
  \segkey{0}{0}{cblue}
  \node[font=\tiny, anchor=west] at ({0+\segW/2+0.02}, {\segH/2}) {High quality};
  \segkey{1.8}{0}{corange}
  \node[font=\tiny, anchor=west] at ({1.8+\segW/2+0.02}, {\segH/2}) {Degraded};
  \segkey{3.4}{0}{cteal}
  \node[font=\tiny, anchor=west] at ({3.4+\segW/2+0.02}, {\segH/2}) {Regenerated};

  \draw[-{Stealth[length=2pt]}, gray!50, line width=0.3pt]
    (0, -0.30) -- (2.0, -0.30);
  \node[font=\tiny, text=gray!50, anchor=west] at (2.08, -0.30) {generation steps};
\end{scope}

\end{tikzpicture}%
}
\vspace{-4pt}
\caption{Comparison of decoding strategies on three trajectories ($\tau_1$: high quality throughout; $\tau_2$: high-quality prefix with degraded suffix; $\tau_3$: low quality). Best-of-$N$ fully generates all candidates and selects the best, discarding the rest. Speculative Rejection terminates trajectories when quality drops, saving compute but permanently losing useful prefixes. SRD preserves $\tau_2$'s high-quality prefix and selectively regenerates only its degraded suffix (teal), recovering a trajectory that both baselines waste.}
\label{fig:motiv}
\vspace{-10pt}
\end{figure}
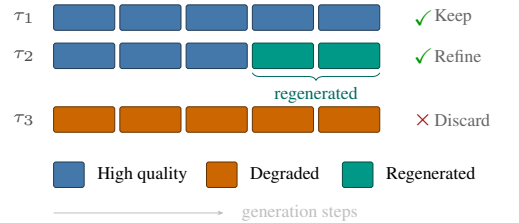

Allocating additional computation at inference time improves the reasoning quality of large language models~\cite{huang-etal-2025-deal,khanov2024args,bon2,specrej,mudgal2023controlled,liao2025rewardguided}.  While ARGS~\cite{khanov2024args}, and DeAL ~\cite{huang-etal-2025-deal} formulate decoding as a general reward-guided search procedure, a growing body of work seeks to improve the accuracy and efficiency of the search. Best-of-$N$~\cite{bon1,bon2} generates multiple complete trajectories and selects the highest-scoring one; speculative rejection~\cite{specrej} terminates low-reward prefixes early to reduce cost; Reward-guided Speculative Decoding~\cite{liao2025rewardguided} uses a reward model to accept or reject draft steps from a smaller model; and Controlled Decoding~\cite{mudgal2023controlled} steers token selection via a learned prefix-aware value function.

Despite their differences, these methods treat reasoning trajectories as atomic units. Best-of-$N$ defers all decisions until full trajectories are generated, incurring redundant computation. Speculative rejection makes earlier decisions, but they are irreversible---once a prefix is rejected, any high-quality reasoning it contains is permanently lost. This is particularly wasteful for long-form reasoning, where trajectories frequently contain strong prefixes that degrade only in later steps. A more efficient strategy would preserve what is already good and intervene only where quality drops.

We introduce \emph{Selective Regenerative Decoding} (SRD), a sample-efficient decoding algorithm that enables segment-level intervention within reasoning trajectories (Figure~\ref{fig:motiv}). SRD operates in three phases: it generates multiple candidates, routes each to \textsc{Keep}, \textsc{Refine}, or \textsc{Discard} based on rank-normalized reward scores, and selectively regenerates only the degraded suffixes of borderline candidates using higher-temperature sampling---without requiring a larger target model. By preserving useful prefixes and rewriting only where quality degrades, SRD recovers computation that rejection-based methods waste.

We prove that, under mild assumptions, SRD achieves a provable $1.28$--$1.36\times$ sample efficiency gain over rejection sampling with strictly higher expected best-trajectory quality with the gain growing as the candidate pool grows---characterized in general by $(1 + \rho \cdot p_M/p_H)$---and satisfies formal termination and weak monotonicity guaranties (\S\ref{subsec:theoretical_analysis}). Empirically, across MATH500~\cite{math500}, GPQA Diamond~\cite{gpqa}, HotpotQA~\cite{yang2018hotpotqa}, and AlpacaEval~\cite{alpaca_eval} with multiple generation--reward model combinations, SRD matches Best-of-$N$ accuracy with substantially fewer generated tokens and outperforms speculative rejection in low-compute regimes (\S\ref{sec:main_results}). Ablation studies reveal that the effectiveness of SRD's refinement depends on reward model calibration: conservative global rerouting excels under stable rewards, while local self-comparison is more robust under noisy signals (\S\ref{sec:salvage_policy}).

Because SRD requires only a generative model and a reward model, it is complementary to speculative decoding~\cite{spec-decoding} and prefix value functions~\cite{mudgal2023controlled}, which can be composed with it. Our main contributions are:
\begin{itemize}
    \item A three-phase generation--routing--refinement algorithm that enables segment-level intervention in reasoning trajectories without a target model.
    \item Formal proofs that SRD achieves strictly better sample efficiency and expected trajectory quality than rejection sampling, with termination and monotonicity guarantees.
    \item Empirical demonstration that SRD consistently improves the accuracy--compute tradeoff across four benchmarks spanning mathematical reasoning, multi-hop QA, and instruction following, with ablations characterizing when and why refinement helps.
\end{itemize}

\section{Related Work}
\label{sec:related_works}

Standard decoding methods, beyond greedy decoding,  explore the local neighborhood for every generated token, such as top-k~\cite{topk}, nucleus sampling~\cite{nucleus}, temperature-based token-sampling~\cite{ackley1985learning}, or the local neighborhood of a path of decoded tokens, such as beam-search~\cite{freitag-al-onaizan-2017-beam}, contrastive search~\cite{su2022contrastive}, and Best-of-N~~\cite{bon1,bon2}. While exploring a larger search neighborhood may lead to decoding candidates that are more optimal (w.r.t. a particular objective), such non-adaptive policies represent uninformed search strategies guided by the generation model's priors and language heuristics.

When viewed as a search process in the space of tokens, we can impose informed heuristics to better guide generated trajectories~\cite{och2001efficient, lu2022neurologic, khanov2024args, huang-etal-2025-deal}. For problems that need reasoning beyond what is well encoded in the language model's priors, these reward signals can be used during the forward-pass~\cite{willard2023efficient, wang2023measuring}, with look-ahead~\cite{bertsch2023s, wan2023faithfulness, huang-etal-2025-deal, nakshatri2025constrained}, and in exploring search graphs/trees \cite{yao2023tree, roy2024flap}.

Along these lines, our work relates more closely to methods that consider pruning search trajectories from a generated set. For example, speculative rejecting~\cite{specrej} evaluates top-k beams or candidate paths with a reward model, and discards less promising prefixes early. Unfortunately, its binary rejection strategy can permanently discard promising candidates that have low-rewards at the start. To relax this strict rejection criterion, Reward-Guided Speculative Decoding [RSD;~\cite{liao2025rewardguided}] performs a relaxed variant of speculative decoding~\cite{spec-decoding}, where a draft model proposes a candidate reasoning step and a reward model decides to accept/reject the draft step. If rejected, a larger (and more capable) reasoning model is used to generate the next reasoning step. However, RSD operates on a single search candidate and doesn't explore multiple search candidate.\footnote{We show the sample efficiency of RSD to be lower compared to a speculative version of SRD in \autoref{app:proofs}.} On the other hand, Controlled Decoding (CD), which introduced an inference-time guidance of policy model's token selection by using a learned prefix value function in tandem, is an implicit way to influence the local search neighborhood for each token rather than having an explicit rejection mechanism~\cite{mudgal2023controlled}.

In our work, we first consider generating set of multiple reasoning paths at each reasoning step. Then, an explicit reward model assesses the promise of each candidate in the set, marking them to be either kept, discarded, or refined.  When salvaged, the search candidate undergoes localized surgery to allow exploration of the search neighborhood. Thus, we mitigate the lack of look-ahead in existing approaches and explore the search neighborhood without the need for any target/teacher model (used in speculative decoding and RSD) or a prefix value-function (in CD). Note that both of these elements can thereby be used with SRD.

\section{Approach}
\label{sec:approach}

We present \textsc{Selective Regenerative Decoding (SRD)}, a sample-efficient search algorithm that improves upon standard rejection sampling by introducing a refinement mechanism (without access to a larger target model) for borderline candidates. Unlike rejection sampling, which treats trajectory generation as atomic, our approach
recognizes that trajectories are compositional and that many generated trajectories can contain high-quality prefixes but suffer from degradation in later steps. Thus, discarding such trajectories entirely wastes valuable computation. Instead, we propose to \emph{salvage} these candidates through targeted regeneration.

\subsection{Problem Setting}
\label{subsec:problem_setting}

Consider a discrete action space $\mathcal{A}$ and let $\mathcal{T} = \mathcal{A}^k$ denote the space of trajectories, where each trajectory $\tau = (a_1, a_2, \ldots, a_k)$ is a sequence of $k$ discrete actions. We assume access to:\\[1pt]
$\bullet$ A generative model $\mathcal{G}$ that produces trajectories $\tau \sim \mathcal{G}$.\\
$\bullet$ A reward model $R: \mathcal{T} \rightarrow \mathbb{R}$ that scores trajectory quality.\\[1pt]
Our goal is to identify high-quality trajectories efficiently, minimizing the number of samples required to obtain trajectories exceeding a quality threshold, all without access to a target/teacher model.

\subsection{Algorithm Overview}
\label{subsec:algorithm_overview}

At each search iteration, \textsc{SRD} operates in three phases: \emph{generation}, \emph{routing}, and \emph{refinement}. Algorithm~\ref{alg:srd_search} provides the complete procedure.

\paragraph{Phase 1: Generation.}
We sample $n$ candidate trajectories independently from the generative model:%
\begin{equation}
    \mathcal{C} = \{\tau_1, \tau_2, \ldots, \tau_n\}, \quad \tau_i \sim \mathcal{G}.
\end{equation}%
\paragraph{Phase 2: Routing.}
Each candidate is evaluated using the reward model and assigned a normalized score based on its relative rank. Specifically, we sort candidates by their reward scores in descending order and assign each trajectory a rank $r \in \{0, 1, \ldots, n-1\}$, where $r = 0$ denotes the highest-scoring candidate. We then compute a normalized score:%
\begin{equation}
    u(\tau) = 1 - \frac{r(\tau)}{n-1},
\end{equation}%
yielding $u \in [0, 1]$ with larger values indicating better relative quality. Then, routing decisions are determined by comparing $u(\tau)$ against two thresholds, $\theta_{\mathrm{low}}$ and $\theta_{\mathrm{high}}$, where $0 \leq \theta_{\mathrm{low}} < \theta_{\mathrm{high}} \leq 1$:%
\begin{equation}
    \textsc{Route}(\tau) = 
    \begin{cases}
        \text{Keep}, & u(\tau) \geq \theta_{\mathrm{high}}, \\
        \text{Refine}, & \theta_{\mathrm{low}} < u(\tau) < \theta_{\mathrm{high}}, \\
        \text{Discard}, & u(\tau) \leq \theta_{\mathrm{low}}.
    \end{cases}
\end{equation}%
\paragraph{Phase 3: Regeneration.}
Trajectories routed to \textsc{Refine} undergo a targeted regeneration procedure. For each such trajectory $\tau$, we identify a \emph{regeneration boundary} by evaluating reward scores at fixed intervals of $m < k$ steps and selecting the earliest position $j^*$ at which reward degrades relative to the preceding prefix:%
\begin{multline}
    j^* = \min \bigl\{j \in \{m, 2m, \ldots\}:\\
    R(\tau_{1:j+m}) < R(\tau_{1:j})\bigr\}
\end{multline}%
We then regenerate the suffix $\tau_{j^*+1:k}$ using a higher-temperature sampling strategy, updating the trajectory:
\begin{equation}
    \begin{aligned}
        \tau' &= (\tau_{1:j^*}, \tilde{\tau}_{j^*+1:k}),\\
        \tilde{\tau}_{j^*+1:k}
        &\sim \mathcal{G}_{\mathrm{high-temp}}(\cdot \mid \tau_{1:j^*}).
    \end{aligned}
\end{equation}%
The refined trajectory $\tau'$ is then re-evaluated and ranked jointly with other active candidates. It is retained only if subsequently routed to \textsc{Keep}; otherwise, it is discarded.
\paragraph{Termination Guarantees.}
To ensure termination, we impose two constraints: (i) a maximum regeneration span length $L_{\max}$, and (ii) a maximum number of refinement attempts per trajectory $N_{\mathrm{refine}}$. These bounds guarantee that the algorithm terminates in at most $n \cdot N_{\mathrm{refine}}$ refinement operations.

{\footnotesize
\begin{algorithm}[t]
\caption{\textsc{Selective Regenerative Decoding}}
\label{alg:srd_search}
\begin{algorithmic}[1]
\Require Generative model $\mathcal{G}$, reward model $R$, sample size $n$, thresholds $\theta_{\mathrm{low}}, \theta_{\mathrm{high}}$, scoring interval $m$, max refinements $N_{\mathrm{refine}}$, max span $L_{\max}$
\Ensure Set of accepted trajectories $\mathcal{K}$
\State $\mathcal{C} \gets \{\tau_i \sim \mathcal{G}\}_{i=1}^{n}$ \Comment{Generate candidates}
\State $\mathcal{K} \gets \emptyset$, $\mathcal{R} \gets \emptyset$ \Comment{Initialize kept and refine sets}
\For{$\tau \in \mathcal{C}$}
    \State Compute $u(\tau)$ based on rank in $\mathcal{C}$
    \If{$u(\tau) \geq \theta_{\mathrm{high}}$}
        \State $\mathcal{K} \gets \mathcal{K} \cup \{\tau\}$
    \ElsIf{$u(\tau) > \theta_{\mathrm{low}}$}
        \State $\mathcal{R} \gets \mathcal{R} \cup \{(\tau, 0)\}$ \Comment{Add with refinement count 0}
    \EndIf
\EndFor
\While{$\mathcal{R} \neq \emptyset$}
    \State $(\tau, c) \gets \textsc{Pop}(\mathcal{R})$
    \If{$c \geq N_{\mathrm{refine}}$}
        \State \textbf{continue} \Comment{Max refinements reached}
    \EndIf
    \State $j^* \gets \textsc{FindBoundary}(\tau, R, m)$
    \If{$k - j^* > L_{\max}$}
        \State $j^* \gets k - L_{\max}$ \Comment{Limit regeneration span}
    \EndIf
    \State $\tau' \gets (\tau_{1:j^*}, \tilde{\tau}_{j^*+1:k})$ where $\tilde{\tau}_{j^*+1:k} \sim \mathcal{G}_{\mathrm{high-temp}}(\cdot \mid \tau_{1:j^*})$
    \State Compute $u(\tau')$ based on rank in $\mathcal{K} \cup \mathcal{R} \cup \{\tau'\}$
    \If{$u(\tau') \geq \theta_{\mathrm{high}}$}
        \State $\mathcal{K} \gets \mathcal{K} \cup \{\tau'\}$
    \ElsIf{$u(\tau') > \theta_{\mathrm{low}}$}
        \State $\mathcal{R} \gets \mathcal{R} \cup \{(\tau', c+1)\}$
    \EndIf
\EndWhile
\State \Return $\mathcal{K}$
\end{algorithmic}
\end{algorithm}
}

\subsection{Theoretical Analysis}
\label{subsec:theoretical_analysis}

We establish two main results: SRD has (i) improved sample efficiency compared to rejection sampling, and (ii) higher expected solution quality. We begin by introducing the necessary notation and assumptions.

\noindent \textbf{Notation and Assumptions\quad}
Let the generative distribution $\mathcal{G}$ induce the following partition of probability mass:
\begin{align}
    p_H &= \probP_{\tau \sim \mathcal{G}}\left[u(\tau) \geq \theta_{\mathrm{high}}\right], \label{eq:p_high} \\
    p_M &= \probP_{\tau \sim \mathcal{G}}\left[\theta_{\mathrm{low}} < u(\tau) < \theta_{\mathrm{high}}\right], \label{eq:p_mid} \\
    p_L &= \probP_{\tau \sim \mathcal{G}}\left[u(\tau) \leq \theta_{\mathrm{low}}\right], \label{eq:p_low}
\end{align}
where $p_H + p_M + p_L = 1$. These are population-level quantities; the rank-based score $u(\tau)$ is a consistent finite-sample estimator of them Appendix-~\ref{app:population_limit}.

\begin{assumption}[Refinement Efficacy]
\label{assump:refinement_efficacy}
There exists a constant $\rho \in (0, 1]$ such that for any trajectory $\tau$ with $\theta_{\mathrm{low}} < u(\tau) < \theta_{\mathrm{high}}$, the refinement procedure produces an acceptable trajectory with probability at least $\rho$:%
\begin{equation}
    \probP\left[\textsc{Route}(\textsc{Refine}(\tau)) = \textsc{Keep}\right] \geq \rho.
\end{equation}%
\end{assumption}%
This assumption captures the intuition that refinement is beneficial, i.e. trajectories in the middle tier have salvageable prefixes, and regenerating their suffixes with increased diversity yields acceptable results with non-negligible probability. Results in \autoref{sec:experiments} support this empirically.
\begin{assumption}[Independence]
\label{assump:independence}
Refinement outcomes are independent across trajectories, and each refinement attempt is independent of previous attempts on the same trajectory.
\end{assumption}

\paragraph{Sample Efficiency}
Our first result establishes \textsc{SRD} requires fewer samples than pure rejection sampling to obtain at least one acceptable trajectory with high probability.

\begin{theorem}[Sample Efficiency]
\label{thm:sample_efficiency}
Let $\delta \in (0, 1)$ be a failure probability. Under Assumptions~\ref{assump:refinement_efficacy} and~\ref{assump:independence}, 
(i) pure rejection sampling requires at least $N_{\mathrm{reject}} \geq \frac{\ln(1/\delta)}{p_H}$ samples to obtain at least one acceptable trajectory with probability $\geq 1 - \delta$.
In comparison, (ii) \textsc{SRD} requires at least $N_{\mathrm{SRD}} \geq \frac{\ln(1/\delta)}{p_H + \rho \cdot p_M}$ samples to achieve the same guarantee.
Thus, (iii) the efficiency gain of SRD over rejection sampling is characterized by:
\begin{equation}
    \frac{N_{\mathrm{reject}}}{N_{\mathrm{refine}}} = 1 + \frac{\rho \cdot p_M}{p_H}.
\end{equation}
\end{theorem}

The complete proof is provided in Appendix~\ref{app:proof_theorem_3_3} and emperical evidence is provided in Appendix~\ref{app:rho_details}. We note that the efficiency gain is most pronounced when $p_H$ is small (high-quality trajectories are rare) and $p_M \cdot \rho$ is substantial (many refinable trajectories exist and refinement is effective). In the regime where $\rho \cdot p_M \geq p_H$, \textsc{SRD} achieves at least a factor of 2 improvement in sample efficiency.

\paragraph{Expected Quality Improvement}
Beyond sample efficiency, \textsc{SRD} also improves the expected quality of the best trajectory found.
\begin{assumption}[Stochastic Dominance of Refinement]
\label{assump:stochastic_dominance}
For trajectories in the refinable region, the refinement operation produces outputs whose reward distribution stochastically dominates that of fresh samples from the same region:
\begin{align}
    R(\textsc{Refine}(\tau)) \succeq_{\mathrm{st}} R(\tau') \quad \text{for } \tau, \tau'\\\nonumber
    \text{ with } \theta_{\mathrm{low}} < u(\tau), u(\tau') < \theta_{\mathrm{high}}
\end{align}%
where $\succeq_{\mathrm{st}}$ denotes first-order stochastic dominance. This holds naturally given a larger target model for regeneration, and is motivated here by refinement preserving high-quality prefixes while regenerating problematic suffixes with increased diversity.
\end{assumption}%
\begin{theorem}[Expected Quality Improvement]
\label{thm:quality_improvement}
Let $R_{\max}^{\mathrm{RS}}(n) = \max_{i \in [n]} R(\tau_i)$ denote the best reward from $n$ samples under pure rejection sampling, and let $R_{\max}^{\mathrm{SRD}}(n)$ denote the best reward from \textsc{SRD}. Under Assumptions~\ref{assump:refinement_efficacy}--\ref{assump:stochastic_dominance}:
\begin{equation}
    \mathbb{E}\left[R_{\max}^{\mathrm{SRD}}(n)\right] \geq \mathbb{E}\left[R_{\max}^{\mathrm{RS}}(n)\right] + \Delta_n,
\end{equation}
where $\Delta_n > 0$ for all $n \geq 1$.%
Moreover, under mild regularity conditions on the reward distribution, the improvement satisfies:
\begin{equation}
    \Delta_n = \Omega\left(\frac{p_M \cdot \rho}{n \cdot (1 + p_M \cdot \rho)}\right).
\end{equation}
\end{theorem}%
While the complete proof appears in Appendix~\ref{app:proof_theorem_3_5}, we highlight that the improvement $\Delta_n$ decreases with $n$, which may seem counterintuitive. This reflects the diminishing marginal returns of additional candidates, i.e. as $n$ grows, both algorithms are increasingly likely to find near-optimal trajectories, so the gap between them shrinks. Emperical proof has been provided in Appendix~\ref{app:theorem2_details}.

Finally, \textsc{SRD} satisfies two correctness properties that follow directly from the algorithm construction: it \emph{terminates} in at most $n \cdot N_{\mathrm{refine}}$ refinement operations and $O(n \cdot N_{\mathrm{refine}})$ reward evaluations, and it is \emph{weakly monotone} --- the best kept reward $R_t^*$ is non-decreasing in the number of refinement operations $t$, so refinement never degrades the incumbent solution. Formal statements and proofs are in \autoref{app:proofs} (Propositions~\ref{prop:termination} and~\ref{prop:monotonicity}).

\section{Experiments}
\label{sec:experiments}

\subsection{Experimental Settings}
\label{sec:exp_settings}

\paragraph{Tasks and Metrics.}
We evaluate on four datasets spanning reasoning and instruction-following: \textbf{MATH500}~\cite{math500} (mathematical reasoning, accuracy), \textbf{GPQA Diamond}~\cite{gpqa} (scientific QA, accuracy following \citet{liao2025rewardguided}), \textbf{HotpotQA}~\cite{yang2018hotpotqa} (multi-hop QA, EM and F1), and \textbf{AlpacaEval}~\cite{alpaca_eval} (instruction following, GPT-4o-mini win rate following \citet{specrej}).

\paragraph{Generation and Reward Models.}
We pair task-appropriate generation and reward models, using the ten combinations listed in~\autoref{tab:model_combinations}. Reward models are matched to task type: math- and reasoning-specific models for MATH500 and GPQA, a retrieval-augmented QA model for HotpotQA, and general-purpose preference models for AlpacaEval.
We intentionally decouple generation and reward models throughout, so the evaluation signal is external rather than biased toward the generator's own preferences~\cite{panickssery2024llm}. The same generation model serves as both drafter and editor. Prompt templates are in Appendix ~\ref{app:prompts}.

\begin{table}[t]
\centering
\small
\resizebox{\linewidth}{!}{%
\begin{tabular}{l l}
\toprule
\textbf{Generation Model} & \textbf{Reward Model} \\
\midrule
\multicolumn{2}{l}{\emph{MATH500}} \\
Llama-3.1-8B-Instruct & AceMath-7B-RM \\
Qwen2.5-Math-1.5B-Instruct & AceMath-7B-RM \\
\midrule
\multicolumn{2}{l}{\emph{GPQA DIAMOND}} \\
Llama-3.1-8B-Instruct & Skywork-o1-Open-PRM-7B \\
Qwen3-4B-Instruct & Skywork-o1-Open-PRM-7B \\
\midrule
\multicolumn{2}{l}{\emph{AlpacaEval}} \\
Llama-3.1-8B-Instruct & FsfairX-LLaMA3-RM-v0.1 \\
Llama-3.1-8B-Instruct & RM-Mistral-7B \\
Qwen3-4B-Instruct & FsfairX-LLaMA3-RM-v0.1 \\
Qwen3-4B-Instruct & RM-Mistral-7B \\
\midrule
\multicolumn{2}{l}{\emph{HotpotQA}} \\
Llama-3.1-8B-Instruct & Llama3.1-RAG-Reward-v2 \\
\bottomrule
\end{tabular}
}\vspace{1mm}
\caption{Generation and reward model combinations used for each dataset.}
\label{tab:model_combinations}
\end{table}

\paragraph{Baselines.}
We compare against three inference-time decoding baselines, all sharing SRD's generation model, prompt template, and sampling hyperparameters and differing only in decoding strategy: \textbf{Temperature Sampling} ($N{=}1$, a single trajectory); \textbf{Best-of-$N$ (BoN)}, which fully generates $N$ candidates and selects the highest-reward one; and \textbf{Speculative Rejection (Spec-Rej)}~\cite{specrej}, which scores prefixes during generation and permanently terminates low-reward trajectories. Full baseline descriptions are in Appendix ~\ref{app:baselines}.
We exclude Reward-guided Speculative Decoding~\cite{liao2025rewardguided} from the main comparison because it assumes access to a larger target model for regeneration, whereas SRD and all baselines above operate with only a single generative model and a reward model. We provide a controlled comparison with a speculative variant of SRD under RSD's setting in~\autoref{app:proofs}.

\subsection{Main results}
\label{sec:main_results}

\begin{figure*}[tbh]
    \centering
    \includegraphics[width=\linewidth]{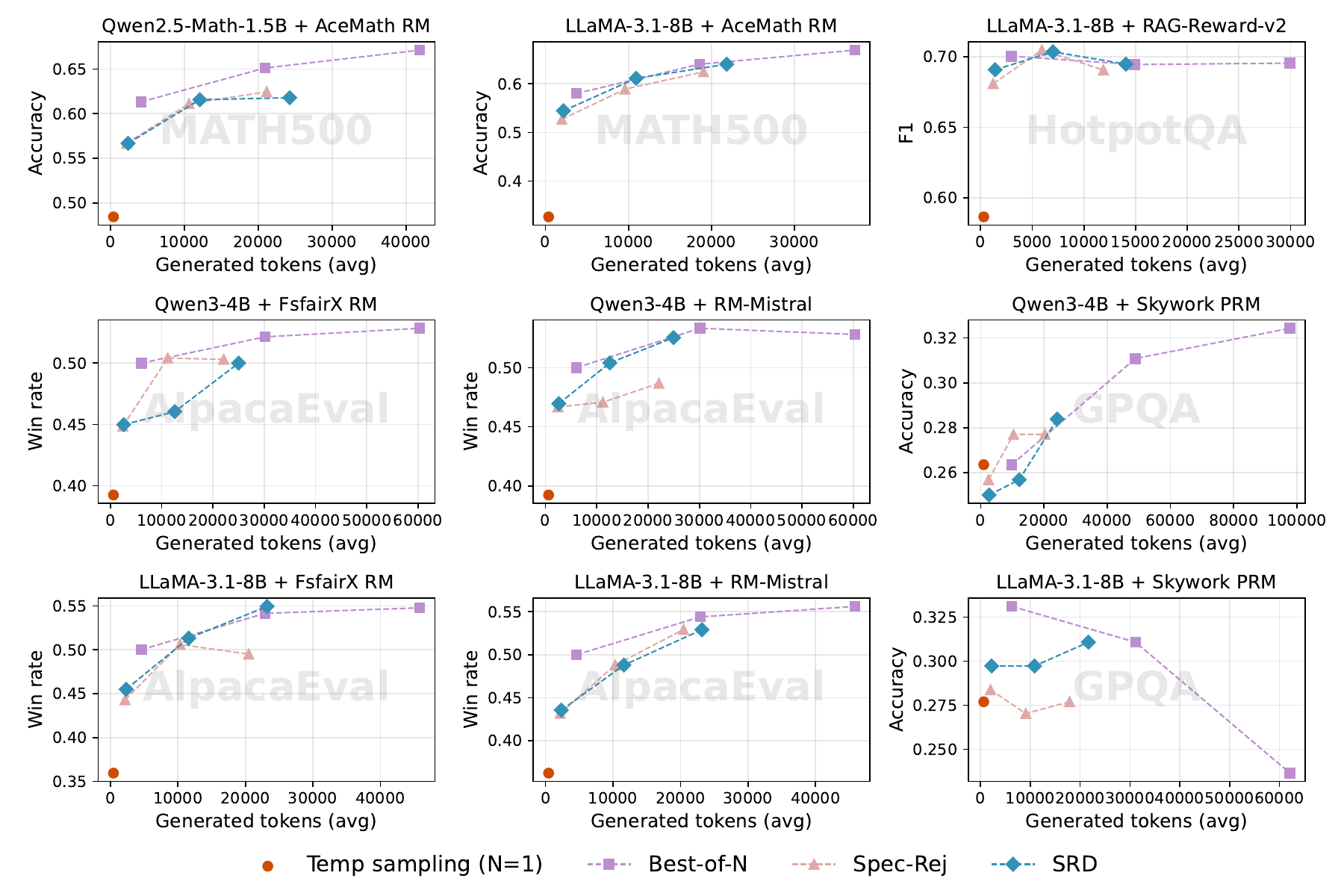}
    \caption{Accuracy-compute trade-offs of inference-time decoding methods across four benchmarks
and multiple generation-reward model combinations.
Compute is measured by the average number of generated tokens from the language model.
SRD consistently achieves competitive performance with respect to Best-of-$N$ and Speculative Rejection under lower or comparable generation budgets by selectively reusing and refining partial trajectories.}
    \label{fig:main_res}
\end{figure*}

\begin{table*}[tbh]
\centering
\resizebox{\linewidth}{!}{
\begin{tabular}{ccccc|ccccc}
\toprule
$N$ & Acc. & Time (s) & In Tok. & Out Tok. &
Drafter (time/calls) & Scorer (time/calls) & Editor (time/calls) & Router (time/calls) \\
\midrule
10  & 0.544 & 7.21  & 6629  & 2166  &
4.73 / 4.1 & 1.29 / 5.6 & 1.18 / 1.0 & \textless 0.01 / 3.8 \\
50  & 0.611 & 18.07 & 32895 & 10927 &
7.43 / 4.2 & 6.98 / 6.6 & 3.59 / 2.2 & \textless 0.01 / 5.3 \\
100 & 0.640 & 29.76 & 65737 & 21840 &
10.96 / 4.2 & 13.53 / 6.7 & 5.13 / 2.3 & \textless 0.01 / 5.5 \\
\bottomrule
\end{tabular}
}\vspace{1mm}
\caption{Component-level breakdown of SRD on MATH500 under different values of $N$.
We report overall accuracy, runtime, token statistics, and per-component time and call frequency
in the format of \textit{time / calls}.}
\label{tab:component_breakdown}
\end{table*}

We evaluate Selective Regenerative Decoding (SRD) on four benchmarks and compare it with several inference-time decoding baselines under an accuracy--compute tradeoff setting.
Throughout this section, we report task-level evaluation metrics
(e.g., accuracy, F1, or win rate), rather than reward scores. All decoding hyperparameters and implementation details are provided in Appendix~\ref{app:details}.

\paragraph{Overall Trends.}
As shown in Figure~\ref{fig:main_res}, SRD consistently forms a distinct tradeoff frontier between the early rejection method (Spec-Rej) and full trajectory sampling (Best-of-N). In particular, SRD achieves accuracy levels comparable to or exceeding Spec-Rej when limited additional computation is permitted.

Rather than operating at a single fixed point, SRD explores a different region of the accuracy--compute tradeoff by enabling segment-level intervention within a single trajectory.
This allows computation to be allocated conditionally, based on intermediate trajectory quality, rather than uniformly across complete candidates.

This trend is particularly evident on reasoning-centric tasks such as MATH500, GPQA, and HotpotQA, where intermediate reasoning errors are common and suggest opportunities for partial trajectory correction.

\paragraph{Mathematical and Reasoning Tasks.}
On \textbf{MATH500} with LLaMA-3.1-8B, SRD achieves 0.544 accuracy at $N{=}10$ using only 2,166 output tokens, scaling to 0.640 at $N{=}100$ with 21,840 tokens (Table~\ref{tab:component_breakdown}). As shown in Figure~\ref{fig:main_res}, these accuracy levels are comparable to Best-of-$N$ at substantially lower token budgets for both Qwen2.5-Math-1.5B and LLaMA-3.1-8B.
While Best-of-$N$ can further improve accuracy by sampling more complete candidates, this comes at the cost of proportionally increased token generation.

On the more challenging \textbf{GPQA} benchmark, overall accuracy is lower across all methods, reflecting task difficulty.
SRD exhibits a more stable accuracy--compute curve in the low- and mid-budget regimes.
Although Best-of-$N$ can achieve higher accuracy under large budgets, its compute cost grows rapidly, indicating diminishing returns from purely increasing candidate diversity.

\paragraph{Multi-hop Question Answering.}
On \textbf{HotpotQA}, SRD attains performance comparable to Best-of-$N$ using a moderate number of generated tokens.
Given that HotpotQA requires multi-step reasoning over multiple pieces of evidence, this result suggests that selectively refining partial trajectories can effectively reduce redundant generation of entire reasoning chains.

\paragraph{Instruction-following Evaluation.}
On \textbf{AlpacaEval}, evaluation is based on preference comparisons and therefore exhibits higher variance.
Despite this, SRD demonstrates a consistent accuracy--compute trend across all four generation--reward model combinations tested (Table~\ref{tab:model_combinations}).
At similar or lower generation budgets, SRD typically achieves win rates comparable to Best-of-$N$
and substantially outperforms single-sample decoding, indicating that SRD remains effective even under noisier reward signals.

\subsection{Empirical Validation of Theoretical Quantities}
\label{sec:empirical_theory}

To directly validate the quantities in Theorem~\ref{thm:sample_efficiency}, we measure the routing distribution and refinement success rate on MATH500 with Qwen2.5-Math-1.5B-Instruct ($N{=}10$).
The accuracy--compute tradeoffs in Figure~\ref{fig:main_res} and the scaling behavior in Table~\ref{tab:component_breakdown} (0.544 at $N{=}10$ to 0.640 at $N{=}100$) provide indirect support for sample efficiency, but do not directly measure $\rho$, $p_H$, or $p_M$; we address this gap here.
We observe $p_H = 0.525$, $p_M = 0.148$, and $p_L = 0.324$, with a refinement efficacy of $\hat{\rho} = 1.0$ (all refinement attempts yielded trajectories routed to \textsc{Keep}).
Substituting into the efficiency gain formula from Theorem~\ref{thm:sample_efficiency} gives $1 + \hat{\rho} \cdot p_M / p_H = 1.28$, indicating that SRD requires approximately 22\% fewer samples than pure rejection sampling to achieve the same success guarantee.
These results are consistent across $N$: at $N{=}20$ and $N{=}30$, we observe $\hat{\rho} = 1.0$ with efficiency gains of $1.35\times$ and $1.36\times$ respectively, as larger candidate pools produce more mid-tier trajectories amenable to refinement.
Extended results across values of $N$ are provided in Appendix~\ref{app:rho_details}.

We additionally validate Theorem~\ref{thm:quality_improvement} by comparing $R_{\max}$ between SRD and rejection sampling under a coupled setting where both methods start from identical initial drafts (via seeded generation).
SRD achieves higher mean $R_{\max}$ than RS at both $N{=}10$ ($\Delta_n = +0.36$) and $N{=}20$ ($\Delta_n = +0.28$), with SRD winning more samples in both cases (23 vs.\ 9 and 27 vs.\ 12, respectively).
Refinement success rates range from 78--100\%, confirming Assumption~\ref{assump:stochastic_dominance}.
Notably, the SRD advantage grows with $N$ ($\Delta_n = +1.29$ at $N{=}30$), as larger candidate pools produce more mid-tier trajectories amenable to refinement.
Full details are provided in Appendix~\ref{app:theorem2_details}.

\subsection{Ablation Studies}
\subsubsection{Effect of Routing Thresholds}
\label{sec:routing_threshold}

We analyze the effect of routing thresholds $(\theta_{\text{high}}, \theta_{\text{low}})$ on SRD using MATH500 with Llama-3.1-8B-Instruct ($N=10$), while keeping all other settings fixed. Table \ref{tab:routing_thresholds} reports representative threshold configurations. Performance degrades when the thresholds are overly aggressive or overly permissive, while the default setting $(0.5, 0.3)$ achieves the best accuracy-compute tradeoff. We use these defaults across all benchmarks and model combinations in the main results (Figure~\ref{fig:main_res}), which provides implicit cross-task validation of this choice.

\begin{table}[t]
\centering
\resizebox{0.8\columnwidth}{!}{
\begin{tabular}{cccc}
\toprule
$\theta_{\text{high}}$ & $\theta_{\text{low}}$ & Accuracy & Out Tokens \\
\midrule
0.5 & 0.3 & \textbf{0.5444} & 2166 \\
0.5 & 0.5 & 0.5400 & 1973 \\
0.5 & 0.2 & 0.5267 & 2325 \\
0.8 & 0.5 & 0.4667 & \textbf{1871} \\
0.7 & 0.3 & 0.4467 & 2091 \\
\bottomrule
\end{tabular}
}\vspace{1mm}
\caption{Effect of routing thresholds on MATH500 (Llama-3.1-8B-Instruct, $N=10$).}
\label{tab:routing_thresholds}
\vspace{-3mm}
\end{table}

\subsubsection{Effect of Refinement Policies}
\label{sec:salvage_policy}

We investigate how different refinement policies---the decision rule applied after regenerating a candidate routed to \textsc{Refine}---affect SRD's behavior.
We vary only the refinement policy and fix all other settings.

\paragraph{Refinement Policy Variants.}
After a candidate routed to \textsc{Refine} undergoes suffix regeneration (Phase~3 in \S\ref{subsec:algorithm_overview}), a policy determines whether the refined candidate is retained. We consider five policies:

\noindent \emph{Reroute (global)} (default) ranks the refined candidate jointly with all active candidates and retains it only if it is routed to \textsc{Keep}.
\emph{Keep on refine} follows the same rerouting procedure but also retains candidates that remain in the \textsc{Refine} tier.
\emph{Self-compare} compares the refined candidate's reward against its pre-refinement reward, retaining it if
$R(\tau') \ge R(\tau) + \delta$, where $\delta$ is a fixed margin ($\delta=0.02$ in this study).
\emph{Force keep} unconditionally retains all refined candidates.
\emph{Refine-BoN} performs $N$ parallel regenerations for each refined candidate and retains the highest-reward result ($N \in \{3, 5, 10\}$).
\begin{table}[t]
\centering
\resizebox{0.8\columnwidth}{!}{
\begin{tabular}{lcc}
\toprule
Refinement Policy & Accuracy & Out Tokens \\
\midrule
\multicolumn{3}{c}{Dataset: MATH500} \\
\midrule
Reroute (global)     & 0.5444 & 2166 \\
Keep on refine       & 0.5333 & 2149 \\
Force keep           & 0.5267 & 2162 \\
Refine-Bo3           & 0.5222 & 2521 \\
Refine-Bo10          & 0.5244 & 3819 \\
Self-compare         & 0.5156 & 2175 \\
Refine-Bo5           & 0.4906 & 2238 \\
\midrule
\multicolumn{3}{c}{Dataset: GPQA} \\
\midrule
Self-compare         & 0.3243 & 2201 \\
Refine-Bo10          & 0.3176 & 4030 \\
Force keep           & 0.3041 & 2212 \\
Reroute (global)     & 0.2973 & 2214 \\
Refine-Bo5           & 0.2635 & 2185 \\
Refine-Bo3           & 0.2635 & 2618 \\
Keep on refine       & 0.2635 & 2185 \\
\bottomrule
\end{tabular}
}\vspace{1mm}
\caption{Effect of different refinement policies on representative reasoning tasks ($N=10$).}
\label{tab:salvage_policy}
\end{table}

\begin{table}[t]
\centering
\resizebox{0.8\columnwidth}{!}{
\begin{tabular}{lcc}
\toprule
Scoring Interval & Accuracy & Out Tokens \\
\midrule
\multicolumn{3}{c}{Dataset: MATH500} \\
\midrule
1   & 0.5533 & 2389 \\
5   & 0.5422 & 2383 \\
10  & 0.5667 & 2387 \\
20  & 0.5556 & 2388 \\
50  & 0.5689 & 2387 \\
100 & 0.5467 & 2390 \\
\midrule
\multicolumn{3}{c}{Dataset: GPQA} \\
\midrule
1   & 0.2365 & 2208 \\
5   & 0.2973 & 2189 \\
10  & 0.2973 & 2214 \\
20  & 0.2500 & 2191 \\
50  & 0.2432 & 2209 \\
100 & 0.2365 & 2218 \\
\bottomrule
\end{tabular}
}\vspace{1mm}
\caption{Ablation on scoring interval (localization granularity), which determines where SRD begins suffix regeneration within a trajectory routed to \textsc{Refine}.}
\label{tab:scoring_interval}
\end{table}

\paragraph{Results.}
Table~\ref{tab:salvage_policy} reports results on MATH500 and GPQA. We use Llama-3.1-8B-Instruct as the generation model; for reward models, we use AceMath-7B-RM on MATH500 and Skywork-o1-Open-PRM on GPQA. On MATH500, where the reward model provides relatively stable and well-calibrated feedback, the conservative \emph{Reroute (global)} strategy achieves the highest accuracy while maintaining low generation cost. More permissive strategies and aggressive refinement variants with internal Best-of-$N$ substantially increase generation without improving accuracy. In contrast, on GPQA, which relies on a noisier process reward model, the locally grounded \emph{Self-compare} strategy achieves the best performance while using fewer generated tokens. Global rerouting-based strategies perform significantly worse, suggesting that global candidate ranking is less reliable under noisy reward signals. Notably, across both regimes, increasing refinement aggressiveness or internal sampling (Refine-BoN) does not improve accuracy and often increases token cost substantially.

\subsubsection{Effect of Scoring Interval}
\label{sec:ablation_scoring_interval}

\paragraph{Mechanism.}
The \textit{scoring interval} controls the granularity at which SRD localizes the regeneration boundary $j^*$ within a trajectory routed to \textsc{Refine} (see Phase~3 in \S\ref{subsec:algorithm_overview}).
Smaller intervals enable finer-grained localization but are more sensitive to local reward fluctuations; larger intervals are more stable but may delay detection of the degradation point.

\paragraph{Observation.}
Across both GPQA and MATH500, as shown in Table \ref{tab:scoring_interval}, extremely small intervals lead to unstable or premature regeneration, while overly large intervals allow errors to propagate before being detected.
A moderate interval consistently yields the most reliable refinement, suggesting a favorable balance between localization precision and reward stability.

\section{Conclusion}

We introduced Selective Regenerative Decoding (SRD), which routes candidate trajectories to keep, refine, or discard and selectively regenerates only degraded suffixes while preserving useful prefixes---achieving a provable $(1 + \rho \cdot p_M/p_H)\times$ sample efficiency gain over rejection sampling. Across four benchmarks, SRD consistently matches Best-of-$N$ accuracy with fewer tokens and outperforms speculative rejection in low-compute regimes, with ablations showing that refinement effectiveness depends critically on reward model calibration. Key limitations include reliance on fixed routing thresholds and heuristic boundary selection; learning these end-to-end, as well as extending SRD to settings with adaptive or learned reward signals, are promising directions for future work.

\section*{Limitations}
SRD requires coordinating multiple components, including a generation model, a reward model, and an editing mechanism, which may increase implementation overhead and inference-time latency in resource-constrained settings.

Moreover, the effectiveness of SRD depends on the quality of the reward model used for routing. Inaccurate or misaligned reward estimates may lead to suboptimal acceptance or salvage decisions.

SRD currently relies on fixed routing thresholds and heuristic boundary selection. Learning these policies end-to-end is an interesting direction for future work.

\section*{Ethical Considerations}

\paragraph{Directed regeneration and bias amplification.}
Unlike Best-of-$N$, which only \emph{selects} among samples the model would have produced anyway, SRD \emph{rewrites} content to raise reward. This gives reward model biases more leverage: systematic preferences over length, style, or assertiveness are actively regenerated toward rather than merely selected for, and repeated refinement can compound the effect. We caution against deploying SRD with a reward model that has not been audited on the target distribution.

\paragraph{Reward models as the optimization target.}
SRD's guarantees are stated with respect to the reward model $R$, not task correctness: weak monotonicity (Proposition~\ref{prop:monotonicity}) ensures the best \emph{scored} trajectory never degrades, which does not imply improved factual accuracy or safety. When $R$ is misaligned with the intended objective, SRD faithfully optimizes the proxy, and the mechanism that recovers a correct suffix can equally entrench a confidently wrong one that the reward model happens to favor. Our GPQA results show this concretely: under a noisier process reward model, trusting global reward rankings is measurably worse than trusting only local comparisons (\S\ref{sec:salvage_policy}). We therefore report task-level metrics rather than reward scores throughout \S\ref{sec:experiments}.

\section*{Impact Statement}

This paper presents Selective Regenerative Decoding (SRD), an inference-time decoding framework aimed at improving the efficiency and reliability of large language model reasoning. By enabling localized refinement within generated trajectories, SRD may help reduce redundant computation and improve performance in long-form reasoning tasks.

However, SRD requires coordinating multiple components, including a generation model, a reward model, and an editing mechanism, which may increase implementation complexity and inference-time latency in resource-constrained settings.

Moreover, SRD relies on reward model feedback for routing decisions. Biases, miscalibration, or misalignment in reward estimates could lead to suboptimal refinements, which is particularly important to consider in high-stakes or safety-critical applications.

SRD currently uses fixed routing thresholds and heuristic boundary selection, and future work may explore more adaptive or learned routing policies. Overall, this work is intended as a general decoding-time contribution, and its broader impacts will depend on the downstream deployment context and the reliability of the reward signals used.

\bibliography{custom}

@article{renyi1953theory,
  title={On the theory of order statistics},
  author={R{\'e}nyi, Alfr{\'e}d},
  journal={Acta Math. Acad. Sci. Hung},
  year={1953}
}

@book{david2004order,
  title={Order statistics},
  author={David, Herbert A and Nagaraja, Haikady N},
  year={2004},
  publisher={John Wiley \& Sons}
}

@article{panickssery2024llm,
  title={Llm evaluators recognize and favor their own generations},
  author={Panickssery, Arjun and Bowman, Samuel and Feng, Shi},
  journal={Advances in Neural Information Processing Systems},
  volume={37},
  pages={68772--68802},
  year={2024}
}

@book{cover1999elements,
  title={Elements of information theory},
  author={Cover, Thomas M},
  year={1999},
  publisher={John Wiley \& Sons}
}

@article{ackley1985learning,
  title={A learning algorithm for Boltzmann machines},
  author={Ackley, David H and Hinton, Geoffrey E and Sejnowski, Terrence J},
  journal={Cognitive science},
  volume={9},
  number={1},
  pages={147--169},
  year={1985},
  publisher={Elsevier}
}

@inproceedings{huang-etal-2025-deal,
    title = "{D}e{AL}: Decoding-time Alignment for Large Language Models",
    author = "Huang, James Y.  and
      Sengupta, Sailik  and
      Bonadiman, Daniele  and
      Lai, Yi-An  and
      Gupta, Arshit  and
      Pappas, Nikolaos  and
      Mansour, Saab  and
      Kirchhoff, Katrin  and
      Roth, Dan",
    editor = "Che, Wanxiang  and
      Nabende, Joyce  and
      Shutova, Ekaterina  and
      Pilehvar, Mohammad Taher",
    booktitle = "Proceedings of the 63rd Annual Meeting of the Association for Computational Linguistics (Volume 1: Long Papers)",
    month = jul,
    year = "2025",
    address = "Vienna, Austria",
    publisher = "Association for Computational Linguistics",
    url = "https://aclanthology.org/2025.acl-long.1274/",
    doi = "10.18653/v1/2025.acl-long.1274",
    pages = "26280--26300",
    ISBN = "979-8-89176-251-0"
}

@inproceedings{khanov2024args,
    title={{ARGS}: Alignment as Reward-Guided Search},
    author={Maxim Khanov and Jirayu Burapacheep and Yixuan Li},
    booktitle={The Twelfth International Conference on Learning Representations},
    year={2024},
    url={https://openreview.net/forum?id=shgx0eqdw6}
}

@article{su2022contrastive,
  title={A contrastive framework for neural text generation},
  author={Su, Yixuan and Lan, Tian and Wang, Yan and Yogatama, Dani and Kong, Lingpeng and Collier, Nigel},
  journal={Advances in Neural Information Processing Systems},
  volume={35},
  pages={21548--21561},
  year={2022}
}

@article{willard2023efficient,
  title={Efficient guided generation for large language models},
  author={Willard, Brandon T and Louf, R{\'e}mi},
  journal={arXiv preprint arXiv:2307.09702},
  year={2023}
}

@inproceedings{roy2024flap,
  title = "{FLAP}: Flow-Adhering Planning with Constrained Decoding in {LLM}s",
    author = "Roy, Shamik  and
      Sengupta, Sailik  and
      Bonadiman, Daniele  and
      Mansour, Saab  and
      Gupta, Arshit",
    editor = "Duh, Kevin  and
      Gomez, Helena  and
      Bethard, Steven",
    booktitle = "Proceedings of the 2024 Conference of the North American Chapter of the Association for Computational Linguistics: Human Language Technologies (Volume 1: Long Papers)",
    month = jun,
    year = "2024",
    address = "Mexico City, Mexico",
    publisher = "Association for Computational Linguistics",
    url = "https://aclanthology.org/2024.naacl-long.29/",
    doi = "10.18653/v1/2024.naacl-long.29",
    pages = "517--539"
}

@inproceedings{nakshatri2025constrained,
  title={Constrained decoding with speculative lookaheads},
  author={Nakshatri, Nishanth Sridhar and Roy, Shamik and Das, Rajarshi and Chaidaroon, Suthee and Boytsov, Leonid and Gangadharaiah, Rashmi},
  booktitle={Proceedings of the 2025 Conference of the Nations of the Americas Chapter of the Association for Computational Linguistics: Human Language Technologies (Volume 1: Long Papers)},
  pages={4681--4700},
  year={2025}
}

@article{yao2023tree,
  title={Tree of thoughts: Deliberate problem solving with large language models},
  author={Yao, Shunyu and Yu, Dian and Zhao, Jeffrey and Shafran, Izhak and Griffiths, Tom and Cao, Yuan and Narasimhan, Karthik},
  journal={Advances in neural information processing systems},
  volume={36},
  pages={11809--11822},
  year={2023}
}

@article{bertsch2023s,
  title={It's MBR all the way down: Modern generation techniques through the lens of minimum Bayes risk},
  author={Bertsch, Amanda and Xie, Alex and Neubig, Graham and Gormley, Matthew R},
  journal={arXiv preprint arXiv:2310.01387},
  year={2023}
}

@article{wan2023faithfulness,
  title={Faithfulness-aware decoding strategies for abstractive summarization},
  author={Wan, David and Liu, Mengwen and McKeown, Kathleen and Dreyer, Markus and Bansal, Mohit},
  journal={arXiv preprint arXiv:2303.03278},
  year={2023}
}

@article{wang2023measuring,
  title={Measuring and mitigating constraint violations of in-context learning for utterance-to-api semantic parsing},
  author={Wang, Shufan and Jean, Sebastien and Sengupta, Sailik and Gung, James and Pappas, Nikolaos and Zhang, Yi},
  journal={arXiv preprint arXiv:2305.15338},
  year={2023}
}

@inproceedings{och2001efficient,
  title={An efficient A* search algorithm for statistical machine translation},
  author={Och, Franz Josef and Ueffing, Nicola and Ney, Hermann},
  booktitle={Proceedings of the ACL 2001 Workshop on Data-Driven Methods in Machine Translation},
  year={2001}
}

@inproceedings{lu2022neurologic,
  title={Neurologic a* esque decoding: Constrained text generation with lookahead heuristics},
  author={Lu, Ximing and Welleck, Sean and West, Peter and Jiang, Liwei and Kasai, Jungo and Khashabi, Daniel and Le Bras, Ronan and Qin, Lianhui and Yu, Youngjae and Zellers, Rowan and others},
  booktitle={Proceedings of the 2022 Conference of the North American Chapter of the Association for Computational Linguistics: Human Language Technologies},
  pages={780--799},
  year={2022}
}

@inproceedings{freitag-al-onaizan-2017-beam,
    title = "Beam Search Strategies for Neural Machine Translation",
    author = "Freitag, Markus  and
      Al-Onaizan, Yaser",
    editor = "Luong, Thang  and
      Birch, Alexandra  and
      Neubig, Graham  and
      Finch, Andrew",
    booktitle = "Proceedings of the First Workshop on Neural Machine Translation",
    month = aug,
    year = "2017",
    address = "Vancouver",
    publisher = "Association for Computational Linguistics",
    url = "https://aclanthology.org/W17-3207/",
    doi = "10.18653/v1/W17-3207",
    pages = "56--60"
}

@article{bon1,
  title={WebGPT: Browser-assisted question-answering with human feedback},
  author={Reiichiro Nakano and Jacob Hilton and Suchir Balaji and Jeff Wu and Ouyang Long and Christina Kim and Christopher Hesse and Shantanu Jain and Vineet Kosaraju and William Saunders and Xu Jiang and Karl Cobbe and Tyna Eloundou and Gretchen Krueger and Kevin Button and Matthew Knight and Benjamin Chess and John Schulman},
  journal={ArXiv},
  year={2021},
  volume={abs/2112.09332},
  url={https://api.semanticscholar.org/CorpusID:245329531}
}

@misc{bon2,
      title={Learning to summarize from human feedback}, 
      author={Nisan Stiennon and Long Ouyang and Jeff Wu and Daniel M. Ziegler and Ryan Lowe and Chelsea Voss and Alec Radford and Dario Amodei and Paul Christiano},
      year={2022},
      eprint={2009.01325},
      archivePrefix={arXiv},
      primaryClass={cs.CL},
      url={https://arxiv.org/abs/2009.01325}, 
}

@inproceedings{
specrej,
title={Fast Best-of-N Decoding via Speculative Rejection},
author={Hanshi Sun and Momin Haider and Ruiqi Zhang and Huitao Yang and Jiahao Qiu and Ming Yin and Mengdi Wang and Peter Bartlett and Andrea Zanette},
booktitle={The Thirty-eighth Annual Conference on Neural Information Processing Systems},
year={2024},
url={https://openreview.net/forum?id=348hfcprUs}
}

@inproceedings{
liao2025rewardguided,
title={Reward-Guided Speculative Decoding for Efficient {LLM} Reasoning},
author={Baohao Liao and Yuhui Xu and Hanze Dong and Junnan Li and Christof Monz and Silvio Savarese and Doyen Sahoo and Caiming Xiong},
booktitle={Forty-second International Conference on Machine Learning},
year={2025},
url={https://openreview.net/forum?id=AVeskAAETB}
}

@article{nucleus,
  title={The curious case of neural text degeneration},
  author={Holtzman, Ari and Buys, Jan and Du, Li and Forbes, Maxwell and Choi, Yejin},
  journal={arXiv preprint arXiv:1904.09751},
  year={2019}
}

@inproceedings{topk,
    title = "Hierarchical Neural Story Generation",
    author = "Fan, Angela  and
      Lewis, Mike  and
      Dauphin, Yann",
    editor = "Gurevych, Iryna  and
      Miyao, Yusuke",
    booktitle = "Proceedings of the 56th Annual Meeting of the Association for Computational Linguistics (Volume 1: Long Papers)",
    month = jul,
    year = "2018",
    address = "Melbourne, Australia",
    publisher = "Association for Computational Linguistics",
    url = "https://aclanthology.org/P18-1082/",
    doi = "10.18653/v1/P18-1082",
    pages = "889--898"
}

@inproceedings{mudgal2023controlled,
title={Controlled Decoding from Language Models},
author={Sidharth Mudgal and Jong Lee and Harish Ganapathy and YaGuang Li and Tao Wang and Yanping Huang and Zhifeng Chen and Heng-Tze Cheng and Michael Collins and Jilin Chen and Alex Beutel and Ahmad Beirami},
booktitle={Socially Responsible Language Modelling Research},
year={2023},
url={https://openreview.net/forum?id=jo57H1CpD8}
}

@inproceedings{spec-decoding,
  title={Fast inference from transformers via speculative decoding},
  author={Leviathan, Yaniv and Kalman, Matan and Matias, Yossi},
  booktitle={International Conference on Machine Learning},
  pages={19274--19286},
  year={2023},
  organization={PMLR}
}

@article{math500,
  title={Measuring Mathematical Problem Solving With the MATH Dataset},
  author={Dan Hendrycks and Collin Burns and Saurav Kadavath and Akul Arora and Steven Basart and Eric Tang and Dawn Song and Jacob Steinhardt},
  journal={NeurIPS},
  year={2021}
}

@inproceedings{gpqa,
  title={Gpqa: A graduate-level google-proof q\&a benchmark},
  author={Rein, David and Hou, Betty Li and Stickland, Asa Cooper and Petty, Jackson and Pang, Richard Yuanzhe and Dirani, Julien and Michael, Julian and Bowman, Samuel R},
  booktitle={First Conference on Language Modeling},
  year={2024}
}

@misc{alpaca_eval,
  author = {Xuechen Li and Tianyi Zhang and Yann Dubois and Rohan Taori and Ishaan Gulrajani and Carlos Guestrin and Percy Liang and Tatsunori B. Hashimoto },
  title = {AlpacaEval: An Automatic Evaluator of Instruction-following Models},
  year = {2023},
  month = {5},
  publisher = {GitHub},
  journal = {GitHub repository},
  howpublished = {\url{https://github.com/tatsu-lab/alpaca_eval}}
}

@inproceedings{yang2018hotpotqa,
  title={{HotpotQA}: A Dataset for Diverse, Explainable Multi-hop Question Answering},
  author={Yang, Zhilin and Qi, Peng and Zhang, Saizheng and Bengio, Yoshua and Cohen, William W. and Salakhutdinov, Ruslan and Manning, Christopher D.},
  booktitle={Conference on Empirical Methods in Natural Language Processing ({EMNLP})},
  year={2018}
}

\newpage
\appendix
\onecolumn

\section{Experimental Settings (Extended)}
\label{app:exp_settings_full}

This appendix expands the experimental settings summarized in \S\ref{sec:exp_settings}.

\subsection{Baseline Descriptions}
\label{app:baselines}

To ensure a fair comparison, unless otherwise specified, all methods share the same generation model, prompt template, and sampling hyperparameters (e.g., temperature), and differ only in their decoding and selection strategies.

\paragraph{Temperature Sampling ($N = 1$).}
This baseline generates a single trajectory under a fixed temperature setting and serves as the simplest sampling-based baseline. When the temperature is set to 0, this method reduces to strictly greedy decoding. In our experiments, we use the same non-zero temperature as other methods to ensure comparability.

\paragraph{Best-of-$N$ (BoN).}
Best-of-$N$ independently generates $N$ full-length candidate sequences and selects the one with the highest reward score. While this approach increases candidate diversity and often improves final output quality, it incurs substantial generation cost, as all $N$ candidates must be fully generated before selection.

\paragraph{Speculative Rejection (Spec-Rej).}
Speculative Rejection evaluates candidate prefixes during generation using a reward model and terminates low-reward trajectories early, thereby reducing generation cost compared to Best-of-$N$. This method adopts a binary accept-or-reject strategy: once a trajectory is rejected, it is permanently discarded and no longer extended.

\section{Component Breakdown and Runtime Analysis}
\label{sec:component_breakdown}

To better understand the computational characteristics of SRD, we provide a component-level
breakdown on MATH500 under different values of $N$ (Table~\ref{tab:component_breakdown}).
We report task performance, runtime, and per-component time and call frequency.
The four components correspond to the phases in \S\ref{subsec:algorithm_overview}:
the \emph{drafter} is the generative model $\mathcal{G}$ that produces candidate trajectories (Phase~1);
the \emph{scorer} is the reward model $R$ used for both routing decisions and regeneration boundary detection (Phases~2--3);
the \emph{router} applies the threshold-based routing logic (Phase~2);
and the \emph{editor} performs suffix regeneration for candidates routed to \textsc{Refine} (Phase~3).

\paragraph{Overall trends.}
As $N$ increases, accuracy consistently improves, while the overall runtime grows moderately.
Importantly, the number of component calls remains largely stable across different $N$,
indicating that SRD does not incur additional structural overhead as candidate exploration increases.
Instead, the increased cost mainly comes from processing longer trajectories rather than invoking
components more frequently.

\paragraph{Cost distribution across components.}
The runtime is dominated by the scorer, followed by the editor, while the router contributes
negligible overhead.
This suggests that SRD's computational cost is primarily driven by reward evaluation and localized
rewriting, whereas routing decisions themselves are lightweight.
Such a cost profile is desirable in practice, as it enables selective intervention without
introducing expensive control logic.

\paragraph{Selective editing behavior.}
Even with larger $N$, the editor is invoked only a small number of times on average,
demonstrating that SRD performs localized edits selectively rather than rewriting entire trajectories.
This behavior distinguishes SRD from naive Best-of-$N$ approaches and helps explain its favorable
accuracy--cost trade-off.

\section{Proofs}
\label{app:proofs}

This appendix provides complete proofs for all theoretical results presented in Section~\ref{subsec:theoretical_analysis}. We begin by clarifying the population-level interpretation of the routing probabilities, then establish several auxiliary lemmas before proving the main theorems.

\subsection{Population-Level Interpretation of \texorpdfstring{$p_H$, $p_M$, $p_L$}{p\_H, p\_M, p\_L}}
\label{app:population_limit}

\begin{remark}
\label{remark:population-level-definition}
The quantities $p_H$, $p_M$, and $p_L$ defined in Equations~\eqref{eq:p_high}--\eqref{eq:p_low} should be interpreted as probabilities under the population distribution -- i.e., the probability that the reward model's score for a random trajectory from $\mathcal{G}$, when compared against the population of all possible outputs, falls in each tier. The algorithm's rank-based scoring $u(\tau) = 1 - r(\tau)/(n-1)$ is a finite-sample approximation to this population quantity. For large $n$, the empirical rank converges to the population quantile by the Glivenko-Cantelli theorem. The theoretical analysis operates in the population limit; the algorithm provides a consistent estimator.
\end{remark}

\subsection{Preliminary Lemmas}

\begin{lemma}[Geometric Failure Probability]
\label{lem:geometric_failure}
Let $X_1, X_2, \ldots, X_N$ be independent Bernoulli random variables with $\Pr[X_i = 1] = p$ for some $p \in (0, 1)$. Then:
\begin{equation}
    \Pr\left[\sum_{i=1}^{N} X_i = 0\right] = (1 - p)^N.
\end{equation}
Furthermore, to ensure $\Pr[\sum_{i=1}^{N} X_i \geq 1] \geq 1 - \delta$ for some $\delta \in (0, 1)$, it is necessary and sufficient that:
\begin{equation}
    \begin{aligned}
        N &\geq \frac{\ln(1/\delta)}{-\ln(1 - p)}\\
          &\geq \frac{\ln(1/\delta)}{p}.
    \end{aligned}
\end{equation}
\end{lemma}

\begin{proof}
The first claim follows directly from independence:
\begin{equation}
    \Pr\left[\sum_{i=1}^{N} X_i = 0\right] = \prod_{i=1}^{N} \Pr[X_i = 0] = \prod_{i=1}^{N} (1 - p) = (1 - p)^N. \nonumber
\end{equation}

For the second claim, we require $(1 - p)^N \leq \delta$. Taking logarithms:
\begin{equation}
    N \ln(1 - p) \leq \ln(\delta), \nonumber
\end{equation}
and since $\ln(1 - p) < 0$ for $p \in (0, 1)$:
\begin{equation}
    N \geq \frac{\ln(1/\delta)}{-\ln(1 - p)} \nonumber
\end{equation}

 Inequality $-\ln(1 - p) \leq p$ follows from the concavity of the logarithm: for $p \in [0, 1)$, we have $\ln(1 - p) \geq -p/(1-p) \geq -p$ when $p \leq 1/2$, and direct computation confirms the bound for $p > 1/2$. Alternatively, this follows from the Taylor expansion $-\ln(1-p) = p + p^2/2 + p^3/3 + \cdots \geq p$. Thus:
\begin{equation}
    \begin{aligned}
        N &\geq \frac{\ln(1/\delta)}{-\ln(1 - p)}\\
          &\geq \frac{\ln(1/\delta)}{p}.
    \end{aligned}
    \nonumber \qedhere
\end{equation}
\end{proof}

\begin{lemma}[Effective Acceptance Probability]
\label{lem:effective_acceptance}
Under Assumptions~\ref{assump:refinement_efficacy} and~\ref{assump:independence}, the probability that a single trajectory sampled from $\mathcal{G}$ is eventually accepted by \textsc{SRD} (either directly or after Regeneration) is at least:
\begin{equation}
    p_{\mathrm{eff}} = p_H + \rho \cdot p_M,
\end{equation}
where $p_H$, $p_M$, and $\rho$ are defined in Equations~\eqref{eq:p_high}--\eqref{eq:p_mid} and Assumption~\ref{assump:refinement_efficacy}, respectively.
\end{lemma}
\begin{proof}
A trajectory $\tau \sim \mathcal{G}$ is eventually accepted if:

(a) $u(\tau) \geq \theta_{\mathrm{high}}$, i.e., it is directly routed to \emph{Keep}. This occurs with probability $p_H$ by Eq.(6).

(b) $\theta_{\mathrm{low}} < u(\tau) < \theta_{\mathrm{high}}$ routed to \emph{Refine}, probability $p_M$ by Eq.(7)), AND the refinement procedure produces an acceptable trajectory (probability $\geq \rho$ by Assumption3.1).

Events (a) and (b) are mutually exclusive since the regions $\{u \geq \theta_{\mathrm{high}}\}$ and $\{\theta_{\mathrm{low}} < u < \theta_{\mathrm{high}}\}$ are disjoint.
By Assumption 3.2 (independence), the refinement outcome in case (b) is independent of the initial routing. Therefore, by the law of total probability:
$$p_{\mathrm{accepted}} \geq p_H + p_M \cdot \rho = p_{\mathrm{eff}}.$$

Note this is a lower bound: with $N_{\mathrm{refine}} > 1$ attempts, the effective probability is $p_H + p_M \cdot (1 - (1-\rho)^{N_{\mathrm{refine}}}) \geq p_H + \rho \cdot p_M$.

\end{proof}

\vspace{0.5em}
\begin{lemma}[Stochastic Dominance of Maxima]
\label{lem:max_dominance}
Let $X$ and $Z$ be independent random variables and define $Y = \max\{X, Z\}$. Then $Y$ stochastically dominates $X$:
\begin{equation}
    Y \succeq_{\mathrm{st}} X,
\end{equation}
with strict dominance (i.e., $\mathbb{E}[Y] > \mathbb{E}[X]$) whenever $\Pr[Z > X] > 0$.
\end{lemma}
\vspace{0.5em}
\begin{lemma}[Order Statistic Gap]
\label{lem:order_stat_gap}
Let $X_1, \ldots, X_n$ be i.i.d. random variables with continuous distribution $F$ having density $f$ bounded away from zero near the right endpoint of its support. Let $X_{(n)}$ and $X_{(n-1)}$ denote the largest and second-largest order statistics. Then:
\begin{equation}
    \mathbb{E}[X_{(n)} - X_{(n-1)}] = \Theta\left(\frac{1}{n}\right).
\end{equation}
\end{lemma}

\begin{proof}
This is a classical result in order statistics. By the R{\'e}nyi representation theorem \cite{renyi1953theory}, if $U_{(1)} \leq \cdots \leq U_{(n)}$ are order statistics of $n$ uniform random variables on $[0, 1]$, then:
\begin{equation}
    (U_{(1)}, \ldots, U_{(n)}) \stackrel{d}{=} \left(\frac{S_1}{S_{n+1}}, \frac{S_2}{S_{n+1}}, \ldots, \frac{S_n}{S_{n+1}}\right), \nonumber
\end{equation}
where $S_k = E_1 + \cdots + E_k$ and $E_1, \ldots, E_{n+1}$ are i.i.d. Exponential(1)/standard-exponential random variables. For the uniform distribution:
\begin{equation}
    \mathbb{E}[U_{(n)} - U_{(n-1)}] = \mathbb{E}\left[\frac{E_n}{S_{n+1}}\right] = \frac{1}{n+1}, \nonumber
\end{equation}
using the fact that $\mathbb{E}[E_i/S_{n+1}] = 1/(n+1)$ by symmetry.

For a general distribution $F$ with density $f$, we use the probability integral transform: $X_{(k)} = F^{-1}(U_{(k)})$. By a Taylor expansion near the right endpoint:
\begin{equation}
    \mathbb{E}[X_{(n)} - X_{(n-1)}] \approx \frac{1}{f(F^{-1}(1-))} \cdot \mathbb{E}[U_{(n)} - U_{(n-1)}] = \Theta\left(\frac{1}{n}\right), \nonumber
\end{equation}
provided $f$ is bounded away from zero. Please see \citet{david2004order} for a complete treatment.
\end{proof}

\subsection{Proof of Theorem~\ref{thm:sample_efficiency} (Sample Efficiency)}
\label{app:proof_theorem_3_3}

\begin{proof}
We prove each statement in the theorem separately.

\paragraph{Part (i): Lower bound for rejection sampling.}

In pure rejection sampling, a trajectory $\tau$ is accepted if and only if $u(\tau) \geq \theta_{\mathrm{high}}$. By definition, this occurs with probability $p_H$.

Let $X_i = \mathbf{1}[u(\tau_i) \geq \theta_{\mathrm{high}}]$ for $i = 1, \ldots, N$, where $\tau_i \sim \mathcal{G}$ are i.i.d. samples. Then $\{X_i\}$ are i.i.d. Bernoulli($p_H$) random variables.

The probability of obtaining at least one accepted trajectory is:
\begin{equation}
    \Pr\left[\sum_{i=1}^{N} X_i \geq 1\right] = 1 - \Pr\left[\sum_{i=1}^{N} X_i = 0\right] = 1 - (1 - p_H)^N. \nonumber
\end{equation}

To ensure this probability is at least $1 - \delta$, we require:
\begin{equation}
    1 - (1 - p_H)^N \geq 1 - \delta \quad \Longleftrightarrow \quad (1 - p_H)^N \leq \delta. \nonumber
\end{equation}

By Lemma~\ref{lem:geometric_failure}, this requires:
\begin{equation}
    N \geq \frac{\ln(1/\delta)}{p_H}. \nonumber
\end{equation}

\paragraph{Part (ii): Lower bound for \textsc{SRD}.}

By Lemma~\ref{lem:effective_acceptance}, the effective acceptance probability under \textsc{SRD} is at least $p_{\mathrm{eff}} = p_H + \rho \cdot p_M$. Following the same argument as Part (i), but with $p_H$ replaced by $p_{\mathrm{eff}}$:
\begin{equation}
    \Pr[\text{at least one acceptance}] = 1 - (1 - p_{\mathrm{eff}})^N \geq 1 - \delta \nonumber
\end{equation}
requires:
\begin{equation}
    N \geq \frac{\ln(1/\delta)}{p_{\mathrm{eff}}} = \frac{\ln(1/\delta)}{p_H + \rho \cdot p_M}. \nonumber
\end{equation}

\paragraph{Part (iii): Efficiency ratio.}

Given the two above statements, we can simply take a ratio of the statements to show the efficiency gain of SRD over Rejection Sampling:
\begin{equation}
    \frac{N_{\mathrm{reject}}}{N_{\mathrm{refine}}} = \frac{\ln(1/\delta) / p_H}{\ln(1/\delta) / (p_H + \rho \cdot p_M)} = \frac{p_H + \rho \cdot p_M}{p_H} = 1 + \frac{\rho \cdot p_M}{p_H}. 
    \label{eq:eff_gain}
    \qedhere
\end{equation}
\end{proof}

\begin{remark}[Tightness of the Bound]
The bounds in Theorem~\ref{thm:sample_efficiency} are tight up to constant factors. Precisely, we can use Fano's inequality \cite{cover1999elements} in information theory to show that any algorithm that achieves a success probability of $1 - \delta$ must collect sufficient information to distinguish between the hypothesis that \emph{at least one good trajectory exists} and \emph{no good trajectory exists}, and this requires $\Omega(\ln(1/\delta) / p_{\mathrm{eff}})$ samples. Hence,
combining the bounds we have:
\begin{enumerate}[label=(\roman*)]
    \item Rejection sampling is optimal among algorithms that treat each sample as either \emph{accept} or \emph{reject} with no intermediate processing.
    \begin{equation}
        \frac{c_1\ln(1/\delta)}{p_H} \leq N_{reject} \leq \frac{c_2\ln(1/\delta)}{p_H}
    \end{equation}

    \item SRD is optimal among algorithms that can additionally salvage borderline samples with success probability $\rho$.
    \begin{equation}
        \frac{c'_1\ln(1/\delta)}{p_{\mathrm{eff}}} \leq N_{refine} \leq \frac{c'_2\ln(1/\delta)}{p_{\mathrm{eff}}}
    \end{equation}

    \item The efficiency gain in \autoref{eq:eff_gain} is real and guaranteed since any algorithm with effective acceptance probability of $p_{\mathrm{eff}}$ will have a sample complexity of $\Theta(ln(1/\delta)/p_{\mathrm{eff}})$, and increasing $p_{\mathrm{eff}}$ by salvaging and regeneration will reduce the samples.
\end{enumerate}
\end{remark}

\paragraph{Corollary} The sample efficiency of Speculative Selective Regenerative Decoding (S-SRD) is higher than that of Reward-guided Speculative Decoding (RSD) \cite{liao2025rewardguided}. Assuming a larger target model is more accurate on the reasoning problem at hand, we can guarantee that $p_{\mathrm{eff}} \geq p_H$ as any salvaged trajectory (provided the draft model can generate this with non-zero probability) regenerated using a more-accurate target model will be accepted.

\subsection{Proof of Theorem~\ref{thm:quality_improvement} (Expected Quality Improvement)}
\label{app:proof_theorem_3_5}

\begin{proof}
The proof proceeds in three steps: (1) constructing a coupling, (2) establishing stochastic dominance, and (3) quantifying the improvement.

\paragraph{Step 1: Coupling construction.}

We construct a coupling between rejection sampling and \textsc{SRD} as follows. Both algorithms receive the same sequence of $n$ initial trajectories $\tau_1, \ldots, \tau_n \sim \mathcal{G}$.

Now, let $\mathcal{A}$ denote the set of accepted trajectories. For rejection sampling, we have:
\begin{equation}
    \mathcal{A}_{\mathrm{reject}} = \{\tau_i : u(\tau_i) \geq \theta_{\mathrm{high}}\}, \nonumber
\end{equation}%
For \textsc{SRD}, the accepted set is:
\begin{equation}
    \mathcal{A}_{\mathrm{regenerate}} = \mathcal{A}_{\mathrm{reject}} \cup \mathcal{A}_{\mathrm{regenerate\cap accept}}, \nonumber
\end{equation}
where $\mathcal{A}_{\mathrm{regenerated}}$ contains trajectories that were regenerated and subsequently accepted. Under this coupling, $\mathcal{A}_{\mathrm{reject}} \subseteq \mathcal{A}_{\mathrm{regenerate}}$ by construction.

\paragraph{Step 2: Stochastic dominance.}

Define:
\begin{align}
    X &= R_{\max}^{\mathrm{RS}}(n) = \max_{\tau \in \mathcal{A}_{\mathrm{RS}}} R(\tau), \nonumber \\
    Z &= \max_{\tau \in \mathcal{A}_{\mathrm{regenerate\cap accept}}} R(\tau). \nonumber
\end{align}

By construction:
\begin{equation}
    R_{\max}^{\mathrm{SRD}}(n) = \max\{X, Z\}. \nonumber
\end{equation}%
By Lemma~\ref{lem:max_dominance}, $\max\{X, Z\} \succeq_{\mathrm{st}} X$. Moreover, strict dominance holds because:
\begin{equation}
    \Pr[Z > X] \geq \Pr[\mathcal{A}_{\mathrm{regenerate\cap accept}} \neq \emptyset] \cdot \Pr[Z > X \mid \mathcal{A}_{\mathrm{regenerate\cap accept}} \neq \emptyset] > 0. \nonumber
\end{equation}%
The first factor is at least $1 - (1 - p_M \cdot \rho)^n > 0$ by Assumption~\ref{assump:refinement_efficacy}. The second factor is positive under Assumption~\ref{assump:stochastic_dominance}, which ensures regenerated trajectories have non-trivial probability of achieving high rewards. Therefore, $\mathbb{E}[R_{\max}^{\mathrm{SRD}}(n)] > \mathbb{E}[R_{\max}^{\mathrm{RS}}(n)]$, establishing $\Delta_n > 0$.

\paragraph{Step 3: Quantifying $\Delta_n$.}

To derive the explicit bound, we analyze the probability that a regenerated trajectory achieves the maximum reward. For this, let $N_{RS} = |\mathcal{A}_{\mathrm{RS}}|$ be the number of directly accepted trajectories and $N_{R+A} = |\mathcal{A}_{\mathrm{regenerate\cap accept}}|$ be the number of successfully regenerated trajectories. By linearity of expectation:
\begin{align}
    \mathbb{E}[N_{RS}] &= n \cdot p_H, \nonumber \\
    \mathbb{E}[N_{R+S}] &\geq n \cdot p_M \cdot \rho. \nonumber
\end{align}
Under Assumption~\ref{assump:stochastic_dominance}, the rewards of trajectories in $\mathcal{A}_{\mathrm{SRD}}$ are exchangeable, the original accepted trajectories in ($\in \mathcal{A}_{RS}$) and regenerated trajectories that are eventually accepted ($\in \mathcal{A}_{regenerate\cap accept}$) are i.i.d as per the reward function. Now, by symmetry, each trajectory in $\mathcal{A}_{\mathrm{SRD}}$ has probability approximately $\frac{1}{|\mathcal{A}_{\mathrm{SRD}}|}$ of being the maximum.

The probability that the maximum comes from $\mathcal{A}_{\mathrm{regenerate\cap accept}}$ rather than $\mathcal{A}_{\mathrm{RS}}$ is approximately:
\begin{equation}
    \Pr[\text{max from regenerated}] \approx \frac{\mathbb{E}[N_{R+S}]}{\mathbb{E}[N_H] + \mathbb{E}[N_{R+S}]} = \frac{n \cdot p_M \cdot \rho}{n \cdot p_H + n \cdot p_M \cdot \rho} = \frac{p_M \cdot \rho}{p_H + p_M \cdot \rho}.
    \label{eq:max-from-ra}
\end{equation}

When a regenerated trajectory achieves the maximum, the improvement over the rejection sampling maximum is, by Lemma~\ref{lem:order_stat_gap}:
\begin{equation}
    \mathbb{E}[R_{\max}^{\mathrm{SRD}}(n) - R_{\max}^{\mathrm{RS}}(n) \mid \text{max from regenerated}] = \Theta\left(\frac{1}{n}\right)
    \label{eq:ra-improvement}
\end{equation}

\noindent Finally, Combining \autoref{eq:max-from-ra} and \autoref{eq:ra-improvement}, we have
\begin{align}
    \Delta_n &= \mathbb{E}[R_{\max}^{\mathrm{SRD}}(n)] - \mathbb{E}[R_{\max}^{\mathrm{RS}}(n)] \nonumber \\
    &\geq \Pr[\text{max from regenerated}] \cdot \mathbb{E}[\text{gap} \mid \text{max from regenerated}] \nonumber \\
    &= \frac{p_M \cdot \rho}{p_H + p_M \cdot \rho} \cdot \Theta\left(\frac{1}{n}\right) \nonumber \\
    &= \Omega\left(\frac{p_M \cdot \rho}{n \cdot (p_H + p_M \cdot \rho)}\right) \nonumber \\
    &= \Omega\left(\frac{p_M \cdot \rho}{n \cdot (1 + p_M \cdot \rho)}\right) \nonumber
\end{align}
where the last line uses $p_H \leq 1$. This completes the proof.
\end{proof}

\subsection{Termination}
\label{app:termination}

\begin{proposition}[Termination]
\label{prop:termination}
\textsc{SRD} terminates in at most $n \cdot N_{\mathrm{refine}}$ refinement operations and $O(n \cdot N_{\mathrm{refine}})$ reward evaluations.
\end{proposition}

\begin{proof}
We show that \textsc{SRD} terminates in finite time by bounding the number of Regeneration operations.

Each of the $n$ initial trajectories can be regenerated at most $N_{\mathrm{refine}}$ times (by the counter $c$ in Algorithm~\ref{alg:srd_search}). Once a trajectory has been regenerated $N_{\mathrm{refine}}$ times, it is never added back to the Regeneration queue $\mathcal{R}$.

Furthermore, each Regeneration operation produces at most one new trajectory, which inherits the Regeneration count of its parent (incremented by one). Thus, the total number of trajectories ever created is at most:
\begin{equation}
    n + n \cdot N_{\mathrm{refine}} = n(1 + N_{\mathrm{refine}}). \nonumber
\end{equation}

Since the Regeneration queue $\mathcal{R}$ starts with at most $n$ trajectories and each trajectory exits the queue after at most $N_{\mathrm{refine}}$ Regeneration attempts, the while loop in Algorithm~\ref{alg:srd_search} executes at most $n \cdot N_{\mathrm{refine}}$ iterations.

Each iteration involves:
\begin{enumerate}
    \item Finding the regeneration boundary: $O(k/m)$ reward evaluations.
    \item Generating the regenerated suffix: $O(1)$ call to the generative model.
    \item Computing the rank of the regenerated trajectory: $O(|\mathcal{K}| + |\mathcal{R}|)$ comparisons.
\end{enumerate}

The total number of reward evaluations is thus:
\begin{equation}
    O\left(n \cdot N_{\mathrm{refine}} \cdot \frac{k}{m}\right) = O(n \cdot N_{\mathrm{refine}}), \nonumber
\end{equation}
since $k/m$ is a constant based on predefined hyperparameter choice.
\end{proof}

\subsection{Weak Monotonicity}
\label{app:monotonicity}

\begin{proposition}[Weak Monotonicity]
\label{prop:monotonicity}
Let $\mathcal{K}_t$ denote the set of kept trajectories after $t$ refinement operations, and let $R_t^* = \max_{\tau \in \mathcal{K}_t} R(\tau)$ (with $R_t^* = -\infty$ if $\mathcal{K}_t = \emptyset$). Then the sequence $\{R_t^*\}_{t \geq 0}$ is non-decreasing.
\end{proposition}

\begin{proof}
Let $\mathcal{K}_t$ denote the set of kept trajectories after $t$ Regeneration operations, and define $R_t^* = \max_{\tau \in \mathcal{K}_t} R(\tau)$ (with $R_t^* = -\infty$ if $\mathcal{K}_t = \emptyset$). We show that $R_{t+1}^* \geq R_t^*$ for all $t \geq 0$.

Consider the $(t+1)$-th regeneration operation. Let $\tau$ be the trajectory being regenerated, and let $\tau'$ be the resulting regenerated trajectory. We have $3$ cases:

\noindent \textbf{Case 1: $\tau'$ is routed to \textsc{Keep}.}
Then $\mathcal{K}_{t+1} = \mathcal{K}_t \cup \{\tau'\}$, so $R_{t+1}^* = \max_{\tau'' \in \mathcal{K}_{t+1}}$ $R(\tau'') = \max\left\{R_t^*, R(\tau')\right\} \geq R_t^*$.

\noindent \textbf{Case 2: $\tau'$ is routed to \textsc{regenerate}.}
Then $\mathcal{K}_{t+1} = \mathcal{K}_t$, so $R_{t+1}^* = R_t^*$.

\noindent \textbf{Case 3: $\tau'$ is routed to \textsc{Discard}.}
Then $\mathcal{K}_{t+1} = \mathcal{K}_t$, so $R_{t+1}^* = R_t^*$.

\noindent In all cases, $R_{t+1}^* \geq R_t^*$, establishing weak monotonicity.
\end{proof}

\section{Prompt Templates}
\label{app:prompts}

This appendix provides the prompt templates used to generate model responses for different benchmarks.
All prompts are used exclusively for \emph{generation} and are shared across all decoding methods.
\subsection{HotpotQA}

\begin{tcolorbox}[title=HotpotQA Generation Prompt]
Answer the following question:

\{question\}

Bear in mind that your response should be strictly based on the following passages:

\{passages\}

\textbf{Your task:}
\begin{enumerate}
  \item Carefully reason step by step, using only the information from the retrieved documents.
  \item Explicitly cite the supporting evidence by referencing the corresponding document IDs and sentence indices.
  \item Ensure that each reasoning step is verifiable from the given documents.
  \item After reasoning, provide the final answer clearly in the format:
\end{enumerate}

\[
\boxed{\text{your final answer here}}
\]
\end{tcolorbox}

\subsection{MATH500}

\begin{tcolorbox}[title=MATH500 Generation Prompt]
Please reason step by step, and put your final answer within \texttt{\textbackslash boxed\{\}}.

\{question\}
\end{tcolorbox}

\subsection{GPQA Diamond}

\begin{tcolorbox}[title=GPQA Diamond Generation Prompt]
Please reason step by step, and choose the correct answer from A/B/C/D inside \texttt{\textbackslash boxed\{\}}.

\{question\}
\end{tcolorbox}

\noindent
\textbf{AlpacaFarm.}
For AlpacaFarm, we directly use the original user instruction as the generation input, without any additional task-specific prompt templates or formatting.

\section{Implementation Details}
\label{app:details}

This appendix summarizes the hyperparameter settings and implementation details used in our experiments.
Unless otherwise specified, all methods share the same generation model, prompt template, and generation-side hyperparameters.

\paragraph{Generation Settings.}
We use temperature-based sampling for all methods, with temperature set to 0.8, top-$p$ set to 0.9, and top-$k$ set to 50.
The maximum generation length is set to 500 tokens for MATH500 and HotpotQA, and 1000 tokens for GPQA and AlpacaFarm.

\paragraph{SRD Parameters.}
Selective Regenerative Decoding uses fixed routing thresholds $\theta_{\text{low}} = 0.3$ and $\theta_{\text{high}} = 0.5$.
The drafting interval is set to 100 steps, matching the configuration used for Speculative Rejection.
The reward model is invoked every 10 decoding steps.
We cap the maximum number of regeneration steps at 30 and enforce a minimum of one active candidate to prevent premature termination.

\paragraph{Speculative Rejection Parameters.}
For Speculative Rejection, we follow the original paper's recommended setting, using a rejection threshold $\alpha = 0.5$.
All other generation and evaluation settings are identical to those used for SRD and Best-of-$N$.

\paragraph{Reward Model Usage.}
Reward models are queried with a fixed input format that concatenates the task instruction, input context, and the partially or fully generated output.
For methods that perform step-level scoring, including SRD and Speculative Rejection, the most recent reward score is used for routing decisions.
No additional reward normalization is applied. For Best-of-$N$, reward scores are computed on the final completed sequences and used for candidate selection.

\paragraph{Hardware.}
Experiments are run on 8$\times$A100 (40GB) GPUs. Generation uses vLLM, and reward models are executed with HuggingFace Transformers.

\section{Empirical Validation of Theorem \ref{thm:sample_efficiency}}
\label{app:rho_details}

To validate the assumptions underlying Theorem~\ref{thm:sample_efficiency}, we directly measure the routing distribution ($p_H$, $p_M$, $p_L$) and refinement efficacy ($\hat{\rho}$) from SRD runs. For each configuration, we record how candidates are routed after initial scoring and track the outcome of each refinement attempt (i.e., whether a trajectory routed to \textsc{Refine} is subsequently promoted to \textsc{Keep}).

\begin{table}[h]
\centering
\begin{tabular}{lcccccc}
\toprule
$N$ & $p_H$ & $p_M$ & $p_L$ & $\hat{\rho}$ & Efficiency Gain \\
\midrule
10 & 0.522 & 0.148 & 0.329 & 1.00 (100\%) & 1.28$\times$ \\
20 & 0.513 & 0.180 & 0.305 & 1.00 (100\%) & 1.35$\times$ \\
30 & 0.509 & 0.182 & 0.308 & 1.00 (100\%) & 1.36$\times$ \\
\bottomrule
\end{tabular}
\vspace{0.5em}
\caption{Empirical routing distributions and refinement efficacy ($\hat{\rho}$) on MATH500 (Qwen2.5-Math-1.5B, AceMath-7B-RM, 50 samples per $N$). The efficiency gain is computed as $1 + \hat{\rho} \cdot p_M / p_H$ per Theorem~\ref{thm:sample_efficiency}.}
\label{tab:rho_empirical}
\end{table}

The measured $\hat{\rho} = 1.0$ across all values of $N$ shows that all salvage attempts were routed to keep, providing strong empirical support for Assumption~\ref{assump:refinement_efficacy}. The efficiency gain increases with $N$ ($1.28\times \to 1.36\times$) because $p_M$ grows slightly while $p_H$ remains stable, indicating that larger candidate pools produce more mid-tier trajectories amenable to refinement. The routing distribution is consistent across $N$, with $p_H \approx 0.51$--$0.52$, $p_M \approx 0.15$--$0.18$, and $p_L \approx 0.31$--$0.33$. We note that while these measurements directly validate sample efficiency of SRD highlighted in Theorem~\ref{thm:sample_efficiency}, comparing the reward distributions between SRD and rejection sampling to verify Theorem~\ref{thm:quality_improvement} is considered in the next section.

\section{Empirical Validation of Theorem~\ref{thm:quality_improvement}}
\label{app:theorem2_details}

To validate Theorem~\ref{thm:quality_improvement}, we compare the maximum reward $R_{\max}$ achieved by SRD and rejection sampling (RS) under a coupled setting on MATH500 with Qwen2.5-Math-1.5B-Instruct and AceMath-7B-RM.

\paragraph{Coupling via Seeded Generation.}
The proof of Theorem~\ref{thm:quality_improvement} relies on a coupling argument where both SRD and RS process the same initial trajectories. To realize this coupling empirically, we pass a deterministic per-sample seed to vLLM's \texttt{SamplingParams}, ensuring that both methods receive identical initial drafts and identical continuations for shared candidates. This guarantees the superset property: every candidate RS keeps, SRD also keeps, plus SRD additionally refines mid-tier candidates.

\begin{table}[h!]
\centering
\small
\setlength{\tabcolsep}{4pt}
\begin{tabular}{ccccccccc}
\toprule
$N$ & SRD $\bar{R}_{\max}$ & RS $\bar{R}_{\max}$ & $\Delta_n$ & SRD wins & RS wins & Ties & Refine improved & Pre $\to$ Post \\
\midrule
10 & 17.70 & 17.34 & $+0.36 \pm 1.70$ & 46\% & 18\% & 36\% & 78\% & 2.85 $\to$ 5.96 \\
20 & 18.75 & 18.47 & $+0.28 \pm 4.11$ & 54\% & 24\% & 22\% & 98\% & 3.41 $\to$ 6.79 \\
30 & 19.96 & 18.66 & $+1.29 \pm 3.41$ & 52\% & 24\% & 24\% & 100\% & 3.69 $\to$ 6.87 \\
\bottomrule
\end{tabular}
\vspace{0.5em}
\caption{Theorem~\ref{thm:quality_improvement} validation on MATH500 (Qwen2.5-Math-1.5B, AceMath-7B-RM, 50 samples per $N$). Both methods start from identical initial drafts via seeded generation. $\Delta_n = \bar{R}_{\max}^{\text{SRD}} - \bar{R}_{\max}^{\text{RS}}$ (mean $\pm$ std). ``Pre $\to$ Post'' reports mean reward before and after refinement.}
\label{tab:theorem2_full}
\end{table}

\paragraph{RS Baseline.}
The RS baseline uses the same incremental generation loop as SRD---generating $N$ partial drafts, scoring at intermediate checkpoints, and keeping candidates above $\theta_{\text{high}}$---but without refinement. This corresponds to the \texttt{spec\_rej} decoder mode (\texttt{SRDDecoder} with \texttt{editor=None}, $\theta_{\text{low}} = \theta_{\text{high}} = 0.5$), matching the rejection sampling defined in \S\ref{subsec:theoretical_analysis}.

\paragraph{Results.}
As shown in Table~\ref{tab:theorem2_full}, SRD achieves higher mean $R_{\max}$ than RS across all tested values of $N$, consistent with Theorem~\ref{thm:quality_improvement}'s prediction that $\Delta_n > 0$. SRD wins more samples than RS at every $N$ (46\% vs.\ 18\% at $N{=}10$; 54\% vs.\ 24\% at $N{=}20$; 52\% vs.\ 24\% at $N{=}30$). The ties correspond to cases where both methods' best candidate is the same high-scoring trajectory that both keep (these are the percentage of high trajectories that needed no refinement).

Notably, the SRD advantage grows with $N$: $\Delta_n$ increases from $+0.36$ at $N{=}10$ to $+1.29$ at $N{=}30$. This is because larger candidate pools produce more mid-tier trajectories amenable to refinement, and the refined candidates increasingly contribute to $R_{\max}$. Refinement success rates confirm this trend (78\% $\to$ 98\% $\to$ 100\%), with substantial reward improvements across all settings (e.g., 3.69 $\to$ 6.87 at $N{=}30$), providing direct empirical support for Assumption~\ref{assump:stochastic_dominance}.

\end{document}